%% file: gathering-journal-rapport-technique.tex
\documentclass[twocolumn,draft]{svjour3}

\smartqed
\journalname{ }

\usepackage{mathptmx}

\usepackage{latexsym}
\usepackage{amsfonts}
\usepackage{amsmath}
\usepackage{amssymb}
\usepackage{array}
\usepackage{etoolbox} 
\usepackage{color}
\usepackage{dsfont}
\usepackage{xspace}
\usepackage{multirow}
\usepackage{threeparttable}
\usepackage{footnote}
\usepackage{array}
\usepackage[utf8]{inputenc}
\usepackage{enumerate}

\usepackage{subfig} 

\usepackage{colortbl}
\usepackage{arydshln}
\usepackage{tikz}
\usetikzlibrary{arrows}
\usetikzlibrary{automata}
\usetikzlibrary{intersections}
\usetikzlibrary{through}
\usetikzlibrary{calc}
\usetikzlibrary{matrix}
\usetikzlibrary{positioning}
\usetikzlibrary{shapes.symbols}
\usetikzlibrary{shapes.geometric}
\usetikzlibrary{shapes.misc}
\pgfdeclarelayer{background}
\pgfsetlayers{background,main}

\usepackage{float}
\floatstyle{ruled}
\newfloat{Algorithm}{thp}{lop}
\newfloat{Policy}{thp}{lop}

\usepackage{hyphenat}

\usepackage{multicol} 
\usepackage{alltt} 

\tikzset{%
	robot/.style={
		circle,
		inner sep=0pt,
		minimum size=5pt,
		fill=black,
		draw=black
	},
	active_robot/.style={
		robot,
		fill=red!50,
		draw=red,
		thick		
	},
	configuration/.style={
		draw=gray,
		rectangle,
		rounded corners,
		minimum height=1.4cm,
		minimum width=2cm,
		very thin
	},
	transition/.style={
		->,
		>=latex,
		draw=black,
		rounded corners
	},
	move/.style={
		->,
		>=stealth,
		densely dashed,
		thick,
		red
	},
	conf/.style={
		draw=black,
		ellipse,
		thin
	},
	trans/.style={
		->,
		>=stealth,
		looseness=.5,
		auto,
		draw=black
	},
	point/.style={
		draw=gray,
		cross out,
		scale=0.75,
		minimum height=1pt,
		minimum width=1pt,
		ultra thin		
	}
}

\floatstyle{ruled}
\newfloat{algorithm}{thp}{lop}[section]
\floatname{algorithm}{Algorithm}
\newfloat{module}{thp}{lop}[section]
\floatname{module}{Module}

\newcommand{\remove}[1]{}

\newcommand{\Adv}{\ensuremath{\mathit{Adv}}\xspace}

\newcommand{\SEC}[1][implicit]{\ensuremath{\ifstrequal{#1}{implicit}{\mathit{SEC}}{\mathit{SEC}_{#1}}}\xspace}

\newcommand{\Conv}[1]{\ensuremath{\mathit{Conv}({#1})}\xspace}

\newcommand{\Vor}[1]{\ensuremath{\mathit{Voronoi}({#1})}\xspace}
\newcommand{\Vcell}[1]{\ensuremath{\mathit{Vcell}({#1})}\xspace}

\newcommand{\conf}[1][implicit]{\ensuremath{\ifstrequal{#1}{implicit}{\gamma}{\gamma_{#1}}}\xspace}
\newcommand{\CONF}[1][implicit]{\ensuremath{\ifstrequal{#1}{implicit}{\Gamma}{\Gamma_{#1}}}\xspace}

\newcommand{\exec}{\ensuremath{e}\xspace}
\newcommand{\fragment}[3][\exec]{\ensuremath{{#1}[\mbox{${#2}${\normalfont :}${#3}$}]}\xspace}

\newcommand{\AG}[1][implicit]{\ensuremath{\ifstrequal{#1}{implicit}{\mathit{AG}}{\mathit{AG}({#1})}}\xspace}

\newcommand{\dist}[2]{\ensuremath{\mathrm{dist}\left({#1},{#2}\right)}\xspace}
\newcommand{\val}[1]{\ensuremath{\mathit{val}({#1})}\xspace}
\newcommand{\mul}[1]{\ensuremath{\mathit{mul}({#1})}\xspace}

\newcommand{\MaxMult}[1]{\ensuremath{\mathit{MaxMult}({#1})}\xspace}
\newcommand{\mulmax}[1]{\ensuremath{\mu({#1})}\xspace}
\newcommand{\rbetween}[2]{\ensuremath{\#\mathit{onSegment}({#1},{#2})}\xspace}

\newcommand{\nCastles}[1]{\ensuremath{\#\mathit{castle}({#1})}\xspace}
\newcommand{\nBlocked}[1]{\ensuremath{\#\mathit{blocked}({#1})}\xspace}
\newcommand{\sumDist}[1]{\ensuremath{\Sigma({#1})}\xspace}

\newcommand{\E}{\mathbb E}
\newcommand{\PP}{\mathbb P}

\newcommand{\set}[1]{\ensuremath{\mathbf{#1}}\xspace}
\newcommand{\multiset}[1]{\ensuremath{{#1}}\xspace}
\newcommand{\point}[1]{\ensuremath{\mathit{#1}}\xspace}
\newcommand{\algo}[1]{\ensuremath{\mathcal{#1}}\xspace}

\newcommand{\castle}{\point}

\begin{document}

\title{Fault and Byzantine Tolerant Self-stabilizing Mobile Robots Gathering
\thanks{This manuscript considerably extends preliminary results presented as an extended
 	abstract at the DISC 2006 conference \cite{DGM+06}. The current version is under review at Distributed Computing Journal  since February 2012 (in a previous form) and since 2014 in the current form. The most important  results have been  also presented in MAC 2010 organized in Ottawa from August 15th to 17th 2010 http://people.scs.carleton.ca/~santoro/MAC/MAC-2010.html} 
}
\subtitle{--- Feasibility Study ---}

\author{Xavier D\'efago
  \and Maria Gradinariu Potop-Butucaru
  \and Julien Cl\'ement 
	\and St\'ephane Messika
  \and Philippe Raipin-Parv\'edy}

\institute{
	X. D\'efago \at
		School of Information Science, JAIST, Ishikawa, Japan
       \\\email{defago@jaist.ac.jp}
  \and
  	M. Potop-Butucaru \at
   		LIP6, Universit\'e Pierre et Marie Curie, Paris 6, France
       \\\email{maria.potop-butucaru@lip6.fr}
  \and
  	J. Clement, S. Messika \at
  		LRI/Universit\'e Paris Sud, France
       \\\email{jclement, messika@lri.fr}
  \and
  	P. Raipin-Parv\'edy \at
  		France Telecom R\&D, France
       \\\email{philippe.raipin@orange-ft.com}}

\maketitle

\begin{abstract}
  Gathering is a fundamental coordination problem in cooperative
  mobile robotics.  In short, given a set of robots with \emph{arbitrary}
  initial locations and no initial agreement on a global coordinate
  system, gathering requires that all robots, following their
  algorithm, reach the exact same but not predetermined location.
  Gathering is particularly
  challenging in networks where robots are oblivious (i.e.,~stateless)
  and direct communication is replaced by observations on their
  respective locations. Interestingly any algorithm that solves
  gathering with oblivious robots is inherently self-stabilizing
  if no specific assumption is made on the initial distribution of the robots.

  In this paper, we significantly extend the studies of deterministic 
  gathering feasibility under different assumptions 
  related to synchrony and faults (crash and Byzantine). Unlike prior work, we
  consider a larger set of scheduling strategies, such as bounded
  schedulers.
  In addition, we extend our study to the 
  feasibility of probabilistic self-stabilizing gathering in both fault-free 
  and fault-prone environments. 
\end{abstract}

\input{introduction}

\input{model}
\input{specification}
\input{solution-journal}
\input{ftss}

\input{btss}

\input{summary}

\input{conclusion}
\section*{Acknowledgments}
We are grateful to Fran\c{c}ois Bonnet, the editor, and the reviewers for their
insightful and valuable comments.

Research partly supported by JSPS KAKENHI Grants No.\,23500060 and No.\,26330020.

\bibliographystyle{plain}
\bibliography{gathering}

\input{appendix}

\end{document}

%% file: introduction.tex

\section{Introduction}

Many applications of mobile robotics envision groups of mobile robots
self-organizing and cooperating toward the resolution of common
objectives. In many cases, the group of robots is aimed at being
deployed in adverse environments, such as space, deep sea, or after
some natural (or unnatural) disaster. It results that the group must
self-organize in the absence of any prior infrastructure (e.g., no global
positioning), and ensure coordination in spite of the presence of faulty robots and
unanticipated changes in the environment.

The \emph{gathering problem}, also known as the
\emph{Rendez-Vous} problem, is a fundamental coordination problem in
cooperative mobile robotics.  In short, given a set of robots with
arbitrary initial location and no initial agreement on a global
coordinate system, gathering requires that all robots, following their
algorithm, reach the exact same location---one not agreed upon
initially---within a \emph{finite} number of steps, and remain there.

Similar to the Consensus problem in conventional distributed systems,
gathering has a simple definition but the existence of a solution
greatly depends on the synchrony of the systems as well as the nature
of the faults that may possibly occur. In this paper, we investigate
some of the fundamental limits of deterministic and probabilistic
gathering in the face of various synchrony and fault assumptions.

To study the gathering problem, we consider a system model first
defined by Suzuki and Yamashita
\cite{SY99}, and some variants with
various degrees of synchrony. The model represents robots as
points that evolve on a plane. At any given time, a robot can be
either idle or active. In the latter case, the robot observes the
locations of the other robots, computes a target position, and moves
toward it. The time when a robot becomes active is governed by an
activation daemon (scheduler). In the original definition 
of Suzuki and Yamashita,
called the \emph{SYm model}, activations (i.e.,~look--compute--move) are
atomic, and the scheduler is assumed to be fair and distributed, meaning
that each robot is activated infinitely often and that any subset of
the robots can be active simultaneously.  In the CORDA model of Prencipe
\cite{Pre01}, activations are completely asynchronous, for
instance allowing robots to be seen while moving. 
Flocchini~\emph{et~al.}
\cite{FPS12} provide an excellent overview on the subject.

Suzuki and Yamashita
\cite{SY99} proposed a gathering algorithm for
non-oblivious robots in the SYm model. They also proved that gathering can be
solved in systems with three or more oblivious robots, but not in systems with only
two.%
\footnote{With two robots, all configurations are symmetrical and may
  lead to robots endlessly swapping their positions. In contrast, with
  three or more robots, an algorithm can be made such that, at each
  step, either the robots remain symmetrical and they eventually reach
  the same location, or symmetry is broken and this is used to move
  one robot at a time into the same location.
}
Prencipe
\cite{Pre05} studied the problem of gathering in both SYm
and CORDA models. He showed that the problem is impossible without
additional assumptions such as being able to detect the multiplicity
of a location (i.e.,~knowing the number of robots that may
simultaneously occupy that location).
Flocchini~\emph{et~al.}
\cite{FPS+05} proposed a solution to gathering, for
oblivious robots with limited visibility in the CORDA model, where
robots share the knowledge of a common direction (e.g.,~as given by a
compass).  Based on that work, Souissi~\emph{et~al.}~\cite{SDY09}
considered a system in which compasses are not necessarily consistent
initially. Ando~\emph{et~al.}~\cite{AOS+99} proposed a gathering 
algorithm for the SYm model with limited visibility. 
Cohen and Peleg~\cite{CP06} studied
the problem when robots' observations and movements are subject to
errors.

None of the studies mentioned above address the feasibility of gathering 
in fault-prone environments. One of the first steps in this direction was done by
Agmon and Peleg~\cite{AP06}.
They proved that gathering of correct robots (called \emph{weak gathering}
in this paper)  can be achieved in the SYm model even in the face of the
crash of a single robot. Furthermore, they proved
that no deterministic gathering algorithm exists in the SYm model that
can tolerate a Byzantine%
\footnote{A Byzantine robot is a faulty robot that behaves
  arbitrarily, possibly in a way to deliberately prevent the other robots from
  gathering in a stable way.}
robot.  Finally, they considered a stronger model, called \emph{fully
synchronous}, in which all robots are always activated simultaneously,
and showed that weak gathering can be solved in that 
model provided that less than one third of the robots are Byzantine.

\paragraph{Contribution.}
In this paper, we study further the feasibility of gathering in the SYm model
in both fault-free and fault-prone (crash and, to some extent, Byzantine) environments.
In particular,
we consider centralized
schedulers%
\footnote{The rationale for considering a centralized scheduler is that,
  with communication facilities, the robots can synchronize by running
  a mutual exclusion algorithm, such as token passing.  }
(i.e.,~activations occur in mutual exclusion) and bounded schedulers
(i.e.,~between any two consecutive activations of a robot, no other robot is activated more than $k$-times for some finite~$k$).

More specifically, we obtain the following important results.

Firstly, we strengthen an important impossibility result of
Prencipe~\cite{Pre05} by showing that it also holds in strictly stronger models.
In particular, in oblivious fault-free environments without multiplicity,
Prencipe~\cite{Pre05} proved the impossibility of \emph{distinct}%
\footnote{
	\emph{Distinct gathering} is a tighter definition of the gathering problem,
	in which robots are required to have distinct positions initially.
	In contrast, \emph{self-stabilizing gathering} puts no such requirements on
	initial configurations. 
}
gathering \emph{under a fair scheduler}.
We considerably strengthen this result by proving that the same problem remains
impossible under more restrictive schedulers, even under \emph{a 2-bounded centralized scheduler}.
We further prove that the problem of \emph{self-stabilizing} gathering
is impossible even under a \emph{round-robin scheduler}, and this is also conjectured
for distinct gathering.

Secondly, still without multiplicity, we prove that self-stabilizing gathering
can be solved probabilistically under a fair bounded scheduler (with arbitrary
by finite bound) when $n\geq3$ and under an unfair scheduler when $n=2$, by
exhibiting a simple algorithm that solves the problem.

Thirdly, given multiplicity, we prove that gathering can be solved
\emph{deterministically} under a \emph{fair centralized scheduler} even if up
to $n-1$ robots can crash. We then extend the algorithm to prove that
gathering can also be solved \emph{probabilistically} even if the scheduler
is not centralized.

Fourthly, we study the case of Byzantine-tolerance
by extending the range of impossibility results.
Most notably, Agmon and Peleg~\cite{AP06} proved that $(3,1)$-Byzantine gathering is impossible
\emph{deterministically} under a \emph{fair scheduler}. We extend the result
by showing that even \emph{probabilistic} gathering is impossible under
a \emph{round-robin scheduler}. We also prove other impossibility results.

%
%
%
More generally, we show in what situations randomized algorithms can help solve
the problem, and when they cannot. To the best of our knowledge our work%
\footnote{%
	An extended abstract of this work was presented at DISC \cite{DGM+06} in 2006,
	although it has been considerably extended since. Meanwhile, some authors have
	published very insightful results on the problem \cite{IIK+13}.
} 
was the first to investigate the feasibility of probabilistic gathering in both fault-free and
fault-prone systems.%







\paragraph{Structure of the paper.}
The rest of the paper is structured as follows. Section~\ref{sec:model} describes the system model and basic terminology. 
Section~\ref{sec:problem} formally defines 
the gathering problem and recalls important lemmas found in the literature.
Section~\ref{sec:ff-gathering} proposes 
possibility and impossibility results for deterministic 
and probabilistic gathering in fault-free environments. 
Section~\ref{sec:ft-gathering} and~\ref{sec:byzt-gathering} extend the study 
to crash and Byzantine prone environments.
Section~\ref{sec:summary} summarizes the results, and Section~\ref{conclusion} concludes the paper.

%% file: model.tex

\section{Model}
\label{sec:model}
We define the system model used in the paper, as well as define important terminology.
The model we consider is based on the SYm model \cite{SY99}, and most
definitions are due to various authors \cite{SY99,Pre01,AP06}. 

%

\subsection{Robot network}
%
A robot network consists of a finite set $\set{R}= \{r_1, \cdots, r_n\}$ of $n$~dimensionless robots
evolving in a boundless 2D Euclidean space, devoid of any landmarks or
obstacles.

Robots cannot communicate with each other and do not share any notion of a global
coordinate system.
In particular, they have no agreement on a common origin, unit distance, or
directions and orientations of the axis.

\subsection{Robot}
%
%
A robot is modeled as an I/O automaton%
\footnote{
	In the CORDA model \cite{Pre01}, a robot exhibits a
	continuous behavior that can be modeled by an hybrid
	I/O automaton (\cite{LSV03}).
}
(\cite{Lyn96}).

Robots are \emph{oblivious} which means that they do not retain any information
on past actions and observations.
The state of a robot consists only of its current position in the
environment, which is neither directly readable%
\footnote{
	The current position is exclusively available in local coordinates.
}
nor directly writable%
\footnote{
	A robot can change its position only through move operations.
}
by the robot's algorithm.

Robots are \emph{anonymous} in that they are not aware of any distinctive
identity and all of them execute the same algorithm consisting of
cycles of the operations: Observe, Compute, Move.
In the SYm model, the three operations are executed atomically. Thus, for
simplicity, an algorithm is expressed as one or more $\mathit{Observe}$ input
actions with effects Compute and Move, and guarded by a possible precondition.
\begin{center}\small
	$\underbrace{\mathit{Observe}(\multiset{\Pi})}_{\mbox{input}}$
	:: $\left<\mathrm{precondition}\right>$
	$\longrightarrow$ $\underbrace{\left<\mathrm{compute}\right> ; \left<\mathrm{move}\right>}_{\mbox{effect}}$
\end{center}
In this paper, actions being always enabled, the precondition is always set to true.

\begin{itemize}
\item \emph{Observe (input action)}.\\
	The parameter to the action is a set~\set{P} or multiset~\multiset{\Pi} of points
	representing the positions occupied by all robots, as expressed in the private
	coordinate system of the robot making the observation.
	The origin of the private coordinate system corresponds to the current
	position of the robot with arbitrary unit distance and orientation.
	
	When the system is said to be \emph{with multiplicity},%
	\footnote{
		Our definition of multiplicity is sometimes called ``strong multiplicity'', in
		contrast to a weaker definition where robots are only able to distinguish whether
		a given location is occupied by one or by several robots \cite{IIK+13}.
	}
	The observation is a multiset~\multiset{\Pi} of points and the multiplicity of an
  element in \multiset{\Pi} corresponds to the number of robots sharing that location.
	Conversely, when the system is said to be	\emph{without multiplicity}, then
	the observation is a set~\set{P}.
	
	In this paper, robots are assumed to have \emph{unlimited visibility}, in that
	all	robots are part of each other's observation regardless of their respective
	distance.
		
\item \emph{Compute}.\\
  A stateless computation returning a target destination in the private
  coordinate system.
    
  If the algorithm is deterministic, the computation is deterministic and
  depends only on the observation \set{P} (or \multiset{\Pi}).
  In contrast, if the algorithm is probabilistic, the output may additionally
  depend on random choices.
  
\item \emph{Move (effect)}.\\
	Directs the actual motion of the robot toward a designated target destination.
	
  The robot may or may not reach this destination. For every robot~$r$, there
  exists a \emph{reachable distance}~$\delta_r>0$ unknown to $r$, such that, any
  target destination computed within a distance $\delta_r$ from the current
  position is reached in that step. Conversely, if the target is not reachable,
  then $r$ travels at least a distance $\delta_r$.
  This condition is necessary to ensure progress. 
  
  We denote by $\delta=\min\limits_r \delta_r$ the minimal reachable distance.
  We often use $\delta$ in place of each individual $\delta_r$ for simplicity,
  but only as a worst case choice.
  
  When not explicitly specified, the trajectory of the robot is assumed to be
  a straight line to the destination.
\end{itemize}

\subsection{Activations and schedulers}
A scheduler decides, for every configuration, which
subset of the robots is active (i.e., allowed to perform their actions).
%
In this paper we consider the following schedulers:
\begin{itemize}
\item \emph{unfair arbitrary}:
	At each activation, a non-empty subset of robots is activated.
	A non-triviality condition ensures that, infinitely often, a non-faulty
	robot becomes active.
\item \emph{unfair centralized}: The scheduler is unfair (as described
  above) with the additional restriction that at most one
  (i.e.,~exactly one) robot is activated at each activation.
\item \emph{fair arbitrary}: At each activation, any non-empty subset of the
  robots is activated, with the guarantee that every robot becomes
  active infinitely often in an infinite execution.
\item \emph{fair centralized}: The scheduler is fair (see above) with
  the additional guarantee that no more than one (i.e.,~exactly one)
  robot is activated at each activation.
\item \emph{fair $k$-bounded}: The scheduler is fair with the additional
  guarantee that there exists some bound $k$ 
  such that between any two consecutive activations of some robot, no
  other robot is activated more than $k$~times.
  The bound may be known or unknown to the robots. 
  In the sequel we assume that robots do not know the scheduler bound.
\item \emph{round-robin}:
	The scheduler is fair
  1-bounded and centralized. This implies that the robots are activated always in the same sequence.
\item \emph{fully synchronized}: Every robot is active at every activation.
\end{itemize}
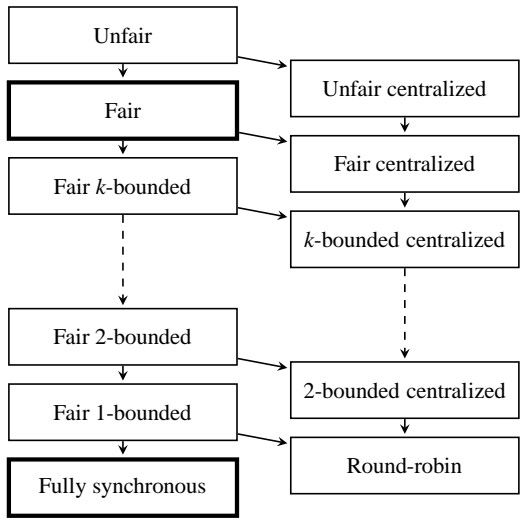
\begin{figure}
		\centering
		\begin{tikzpicture}[
					-stealth,shorten >=1pt,auto,node distance=1cm,on grid,%
	        semithick,inner sep=2pt,bend angle=10,%
	        sched/.style={ state, rectangle, minimum width=3cm }
			]
	    \node[sched] (unfair) {Unfair};
	    \node[sched,ultra thick] (fair)      [below of=unfair]    {Fair};
	    \node[sched] (fkbounded) [below of=fair]      {Fair $k$-bounded};
	    \node (gap) [below of=fkbounded] {};
	    \node[sched] (f2bounded) [below of=gap] {Fair $2$-bounded};
	    \node[sched] (freg)      [below of=f2bounded] {Fair 1-bounded};
	    \node[sched,ultra thick] (fsync)     [below of=freg]      {Fully synchronous};
	    
	    \node[sched] (ucentral)  [below right of=unfair,right=3cm of unfair] {Unfair centralized};
	    \node[sched] (central)   [below of=ucentral] {Fair centralized};
	    \node[sched] (bcentral)  [below of=central]  {$k$-bounded centralized};
	    
	    \node[sched] (b2central) [below right of=f2bounded,right=3cm of f2bounded] {$2$-bounded centralized};
	    \node[sched] (robin)     [below right of=freg,right=3cm of freg] {Round-robin};
	    
	    \path (unfair)    edge (fair)
	    						      edge (ucentral);
	    \path (fair)      edge (fkbounded)  
	    								  edge (central);
	    \path[dashed] (fkbounded) edge (f2bounded);
	    \path (fkbounded) edge (bcentral);
	    \path (f2bounded) edge (freg)
	    						      edge (b2central);
	    \path (freg)      edge (fsync)
	    									edge (robin);
	    
	    \path (ucentral) edge (central);
	    \path (central)  edge (bcentral);
			\path[dashed] (bcentral) edge (b2central);
			\path (b2central) edge (robin);
			
		\end{tikzpicture}
		\caption{\label{fig:schedulers}%
			Relationships between scheduler classes. Conventional models are highlighted: SYm \cite{SY99} and CORDA \cite{Pre01} are fair, and the fully synchronous model \cite{AP06} is its namesake.
		}
\end{figure}
Figure~\ref{fig:schedulers} summarizes the relationships between the schedulers presented
above. Given two schedulers $A$ and $B$, $A\supset B$ means that the
set of all possible executions allowed by scheduler~$A$ strictly
contains the set of all executions allowed by scheduler~$B$. As a
result, any algorithm that is correct under scheduler~$A$ is also
correct under scheduler~$B$. Likewise, any impossibility proven under
scheduler~$B$ also holds under scheduler~$A$.




\subsection{Executions and configurations}
A \emph{configuration} is the union of
the local states of the robots in the system at some discrete time~$t$.
An \emph{execution} $e=(\conf[0], \ldots, \conf[t], \ldots)$ of the system is a
sequence (finite or infinite) of configurations, where \conf[0] is an
initial configuration of the system,
and every transition $\conf[t] \rightarrow \conf[t+1]$ corresponds to the
activation of a subset of the robots, according to the scheduler.
An \emph{execution fragment} is any non-empty subsequence of an
execution.

\medskip
%
The \emph{valence} of a configuration~\conf denotes the number of distinct
locations occupied by some robot in~\conf.
Thus, a \emph{$q$-valent} configuration has $q$ distinct locations (where
$1 \leq q \leq n$ is the valence and $n$ the number of robots in the
system).

A \emph{univalent} configuration is a configuration in which all robots share the same location (valence~1). A univalent configuration~\conf is said to be \emph{centered at $p$} if $p$ is the
location occupied by the robots in \conf.

A \emph{multivalent} configuration is a configuration that is not univalent ($q > 1$).

A \emph{bivalent} configuration is a multivalent configuration with valence~2.

A \emph{1-bivalent} configuration is a bivalent configuration in which one of the two locations is occupied by a single robot.

A \emph{distinct} configuration is a configuration in which all robots have distinct positions (valence~$n$).

\subsection{Fault models}
\label{sec:model:faults}
%
The behavior of a \emph{correct} robot never deviates from its
specification. In contrast, a robot is considered \emph{faulty}
if its behavior deviates from its specification in some executions.
In this paper, we consider two classes of faults: crash and Byzantine.

\paragraph{$(n,f)$-crash model:}
The system consists of $n$ robots, among which up to $f$ faulty robots may
fail by crashing. To rule out the trivial case, $f<n$, so there is at least one
correct robot.

A crash may occur at any time. A robot that crashes permanently stops performing
any action. In particular, it no longer moves from the position it crashed.
A crash cannot be detected by other robots.

\paragraph{$(n,f)$-Byzantine model:}
The system consists of $n$ robots, among which up to $f<n$ faulty robots may
exhibit an arbitrary behavior.

Byzantine robots are controlled by an adversary. The activations of Byzantine
robots are subject to the restrictions imposed by the scheduler. The behavior of
the Byzantine robots can however be based on a global awareness of the environment,
including all past actions and the current state of all robots.

Since a Byzantine robot may elect to stop performing actions, the Byzantine
model is a strict generalization of the crash model.

%
%
%

%
\subsection{Computational Models}
The literature proposes mainly two computational models, namely, SYm and CORDA.
The SYm model was introduced by Suzuki and Yamashita~\cite{SY99}. In this
model each robot performs, once activated by the scheduler, a
\emph{computation cycle} consisting of the following three actions:
observation, computation and motion.  The atomic action performed by a
robot in this model is a computation cycle.  The execution of the
system can be modeled as an infinite sequence of rounds. In a round
one or more robots are activated and perform a computation cycle.

The CORDA model, introduced by Prencipe~\cite{Pre01},
refines the atomicity of actions, by decoupling observe and move actions, as well
as separate the beginning and the end of a move as distinct events.
Robots may be interrupted by the scheduler halfway through a computation
cycle. Moreover, while a robot performs an observation, another robot
may be partway through a movement.

As stated before, in this paper we consider the SYm model,%
\footnote{
	Note that all impossibility results proven in the SYm model
	necessarily hold in the CORDA model.
}
refined with the above scheduling strategies.
We focus our study on the case of oblivious robots, i.e., 
robots do not conserve any
information between two computational cycles.
A major motivation for considering oblivious robots
is that, as observed by Suzuki and Yamashita~\cite{SY99}, 
algorithms designed for that model are inherently self-stabilizing \cite{Dol00}.

\subsection{Notation}
\label{sec:model:notation}
Let \conf be a configuration, then \val{\conf} denotes the valence of configuration~\conf.

Let \multiset{\Pi} be a multiset of points representing the locations of robots in configuration~\conf,
and let \point{p} be a location in \multiset{\Pi}.
Then, \mul{\point{p}} is the multiplicity of point \point{p} and corresponds to the number of robots
located at \point{p} in configuration \conf.

The maximal multiplicity \mulmax{\multiset{\Pi}} (resp. \mulmax{\conf}) of a multiset (resp. configuration)
is $\mulmax{\multiset{\Pi}} = \max\limits_{\point{p}\in \multiset{\Pi}} \left(\mul{\point{p}}\right)$.

We now define the set of points with maximal multiplicity as
$\MaxMult{\multiset{\Pi}}=\left\{ \point{p} \in \multiset{\Pi} ~|~ \mul{\point{p}} = \mulmax{\multiset{\Pi}} \right\}$.
A point in \MaxMult{\conf} is called a point of maximal multiplicity.

For convenience, we introduce the following additional terminology.
A \emph{tower} is a location occupied by at least two robots. 
A \emph{castle} is a tower with maximal multiplicity.

\subsection{Geometry Definitions}
\label{sec:model:geometry}
Given a set of points $\set{P}$, we have the following definitions.

\paragraph{Convex Hull:}
The \emph{convex hull}, denoted \Conv{\set{P}}, is defined as the smallest convex
set that contains \set{P}. The convex hull is unique. A point~\point{p} in \set{P} is a vertex of the convex
hull if and only if \point{p} is outside of \Conv{\set{P}\setminus\{\point{p}\}}.

\paragraph{Smallest Enclosing Circle:}
The \emph{smallest enclosing circle}, denoted $\SEC(\set{P})$, is defined as the smallest
circle that contains all point in \set{P}. It is unique and can be computed in linear
time~\cite{Meg83}. It is defined either by two points which form a diameter, or by three or
more points located on its circumference and forming no angle greater than $\pi$.
Any point in \set{P} on the circumference of $\SEC(\set{P})$ is also a vertex of the convex hull.
The diameter of $\SEC(\set{P})$ provides an upper bound on the distance between any pair of points in \set{P}.

\paragraph{Voronoi Diagram:}
The \emph{Voronoi diagram} \Vor{\set{P}} is a division of the space into
cells, one for each point in \set{P}, such that the Voronoi cell \Vcell{\point{p}} of point \point{p} contains
all points whose distance to \point{p} is smaller or equal to its distance to any other points in \set{P}.
The Voronoi diagram is unique and maps the entire space. All Voronoi cells are convex polygons.
Given a point \point{p} in \set{P}, \Vcell{\point{p}} has vertex at infinity if and only if \point{p} is a vertex of
the convex hull \Conv{\set{P}}.

%
%
%
%
%
%

%% file: specification.tex

%
%

\section{The Self-Stabilizing Gathering Problem}
\label{sec:problem}
In the gathering problem, robots are required to eventually reach a configuration
in which they all share the same location.
There are several variants to the
problem.

\subsection{Strong gathering}
We define the \emph{self-stabilizing strong gathering} problem as follows.
\begin{description}
\item[\textbf{Convergence:}]
	Any execution starting in an arbitrary configuration reaches a univalent configuration
	after a finite number of steps.
	
\item[\textbf{Closure:}]
	Any execution suffix that starts in a univalent configuration contains only univalent
	configurations.
\end{description}

The problem is called \emph{point formation} with an equivalent definition by Suzuki
and Yamashita \cite{SY99}.

\begin{note}
	Other authors, such as Prencipe \cite{Pre05}, define gathering as the problem of reaching
	a univalent configuration when starting from any \emph{distinct} configuration rather than
	arbitrary ones. Let us call that definition ``distinct gathering.''
	Distinct gathering is however not self-stabilizing because, solving the problem
	with oblivious robots does not readily make the algorithm self-stabilizing.
		
	Distinct gathering is covered by self-stabilizing gathering. In other words, an algorithm
	that solves self-stabilizing gathering also solves distinct gathering. Conversely, if
	distinct gathering is impossible in a given system, then self-stabilizing is also impossible
	in that system.
	
	In the paper, we consider the self-stabilizing definition, except in
	Section~\ref{sec:ff-gathering} when we extend impossibility results that were originally
	proved for distinct gathering.
\end{note}

\subsection{Weak gathering}
The definition of strong gathering and univalent does not distinguish between correct robots
and faulty ones. In fault-tolerant contexts, a weaker definition of the problem is often
desirable.

Let us define a \emph{gathered configuration} as a configuration in which all \emph{correct}
robots are located at a unique point of maximal multiplicity.
\begin{description}
\item[\textbf{Convergence:}]
	Any execution starting in an arbitrary configuration reaches a \emph{gathered}
	configuration after a finite number of steps.
	
\item[\textbf{Closure:}]
	Any execution suffix that starts in a \emph{gathered} configuration contains
	only \emph{gathered} configurations.
\end{description}

In a fault-free system, univalent and gathered configurations are identical. Consequently, the
distinction between strong and weak gathering is irrelevant in that context.

\subsection{Convergence}
Gathering is difficult to achieve in most environments. And thus, weaker forms
of gathering were studied so far. An interesting version of this problem requires
robots to \emph{converge} toward a single location rather
than reach that location in a finite time. Convergence is however considerably
easier to deal with. For instance, with unlimited visibility, it can be
achieved trivially by having robots moving toward the barycenter of the
network \cite{SY99}.

\subsection{Existing Results}
We now present a few lemmas proved previously by others, that are
related to our study. When appropriate, the lemmas have been rephrased
in order to keep the terminology consistent. First, the following two
lemmas have been proved by Suzuki and Yamashita~\cite{SY99} and refer
to oblivious robots under a fair scheduler.

\begin{theorem}[\cite{SY99}; Th.~3.1]
  \label{th:pre:nodet2robots}
  There is no deterministic algorithm that solves gathering for $n=2$ robots under a fair scheduler.
\end{theorem}

Notice that, although the above theorem is expressed according to a fair
scheduler (SYm model), the execution used in the proof to show the
impossibility is compatible with a fair bounded scheduler with the bound $k=1$. It
follows that the result also applies to a system based on that scheduler.

\begin{theorem}[\cite{SY99}; Th.~3.4]
  \label{th:pre:pos3robots}
  Gathering of $n\geq 3$ robots can be solved deterministically under a fair scheduler
  with multiplicity detection.
\end{theorem}

The next theorem, proved by Prencipe
\cite{Pre05}, considers distinct gathering (i.e., gathering starting from any
\emph{distinct} configuration) and also applies to oblivious robots under a
fair scheduler.

\begin{theorem}[\cite{Pre05}; Th.~2]
  \label{th:pre:nomultimposs}
  Under a fair scheduler,
  There is no
  deterministic algorithm that solves distinct gathering for $n\geq 2$
  robots without additional assumptions (e.g., multiplicity detection).
\end{theorem}

Finally, the following two theorems, proved by Agmon and Peleg
\cite{AP06}, refer to models with the presence of faulty robots. 
These theorems state positive results.

\begin{theorem}[\cite{AP06}; Th.~3.5]
  \label{th:ap:crashpos}
  Weak gathering can be solved deterministically in a $(3,1)$-crash model
  under a fair scheduler with multiplicity detection.
\end{theorem}

\begin{note}
	\label{note:imp:crash:(n,1)}
	Agmon and Peleg~\cite{AP06} also show (Th.~3.8) that weak gathering can be solved
	by a deterministic algorithm in an $(n,1)$-crash model for any $n \geq 3$, but under
	the restriction that the system is never in a configuration with more than
	one point of multiplicity. 
	
	When $n=3$, there cannot be more than one point of multiplicity, so this is not an
	issue. But, for $n>3$, although their algorithm does solve the \emph{distinct}
	gathering problem, it fails to solve self-stabilizing gathering.
	The definition of the latter problem indeed requires that any configuration leads
	to gathering,	including any one with several points of multiplicity.
\end{note}

%
They also present also two highly relevant results relating to Byzantine models.

\begin{theorem}[\cite{AP06}; Th.~4.4]
  \label{th:ap:byzimposs}
  There is no deterministic
  algorithm that solves weak gathering in a $(3,1)$-Byzantine model under
  a fair scheduler.
\end{theorem}


In contrast, they state a positive result in the fully-syn\-chronous
model---a model in which all robots are activated at every step.
\begin{theorem}[\cite{AP06}; Th.~5.3]
	\label{th:ap:byz3syncposs}
	Weak gathering can be solved deterministically in a $(3,1)$-Byzantine
	system in the fully-synchronous model.
\end{theorem}

\begin{theorem}[\cite{AP06}; Th.~5.10]
	\label{th:ap:byzfsyncposs}
	Weak gathering can be solved deterministically in an $(n,f)$-Byzantine
	system in the fully-synchronous model for any $n\geq 3f+1$.
\end{theorem}


\begin{theorem}[\cite{DP12}; Th.~1]
\label{th:dp:detgathposs}
With strong multiplicity detection, there exists a deterministic algorithm solving self-stabilizing 
gathering in the semi-synchronous model for $n$ robots if, and only if, $n$ is odd.
 \end{theorem}

The following theorem synthetizes the recent results related to the probabilistic gathering under various 
multiplicity conditions. In particular,  \cite{IIK+13}, introduces the notions of local-weak and local-strong multiplicity.
Local multiplicity means that a robot is able to detect the multiplicity only for its current position. Local-weak multiplicity 
means that a robot can detect if at its local position there are one or more than one robots. 
Local-strong multiplicity means that a robot can detect the exact number of robots at its location.
\begin{theorem}[\cite{IIK+13}]
\label{th:iik:probgathposs}
	Probabilistic self-stabilizing gathering is possible in constant expected time 
	with local-strong multiplicity and exponential expected time with local-weak multiplicity.
	Probabilistic distinct gathering is possible in constant expected time with local-weak multiplicity.
\end{theorem}

The next result states the possibility of wait-free%
\footnote{%
	An algorithm is said to be wait-free if it tolerates the crash of up to $n-1$ robots.
}
distinct 
gathering (i.e., the initial configuration must exclude balanced bivalent configurations) in the semi-synchronous model, when robots have 
strong multiplicity detection and chirality knowledge.
\begin{theorem}[\cite{BDT13}]
\label{th:bdt:waitfreegathposs}
	In the semi-synchronous model, wait-free gathering is possible with fair scheduler, under the following assumptions: chirality knowledge and strong multiplicity detection. 
\end{theorem}
The following results refer to the possibility and impossibility of convergence and, by consequence, of gathering, when some robots in the system have Byzantine behavior.

\begin{theorem}[\cite{BPT10}]
\label{th:bpt:convimposs-fs}
	Byzantine-resilient convergence in one-dimensional robot networks is impossible under a fully-syn\-chronous scheduler when $n\leq 2f$.
\end{theorem}

\begin{theorem}[\cite{BPT10}]
\label{th:bpt:convimposs-kbounded}
	Byzantine-resilient convergence In one-dimentional robot networks is impossible under a fair $k$-bounded scheduler ($k>1$) when $n\leq 3f$.
\end{theorem}

\begin{theorem}[\cite{BPT09}]
\label{th:bpt:convimposs-async}
	Starting from a trivalent configuration, no cautious algorithm is able to achieve byzantine-resilient convergence in uni-dimensional networks under an asynchronous scheduler when $3f < n \leq 5f$.
\end{theorem}

%% file: solution-journal.tex

\section{Gathering in Fault-Free Environments}
\label{sec:ff-gathering}
In this section, we refine results showing the impossibility of
gathering~\cite{Pre05,AP06} by proving first that these results hold
even under more restrictive schedulers.
Interestingly, we also prove that some of these impossibility results
hold even in probabilistic settings.  Additionally, to circumvent
these impossibility results, we propose a probabilistic algorithm that
solves the fault-free gathering, under a bounded scheduler. 

First, we introduce two support lemmas that apply to any gathering
algorithm (deterministic or probabilistic) under any form of centralized
scheduler.

\begin{lemma}
	\label{lem:centralized:valence-one}
	Under a centralized scheduler and in any execution,
	the valence of two consecutive configurations
	differs by at most one.
\end{lemma}
\begin{proof}
	The scheduler being centralized, at most one robot is active at each
	step. Regardless of the algorithm, the movement of the active robot
	falls into one of three categories, depending on the respective
	multiplicities of the departure and destination locations of the
	movement:
	\begin{description}
	\item[\textbf{Move~1:}] distinct $\rightarrow$ multiple.\\
	  The valence decreases by one.
	\item[\textbf{Move~2:}] multiple $\rightarrow$ multiple or distinct $\rightarrow$ distinct.\\
		The valence is unchanged.
	\item[\textbf{Move~3:}] multiple $\rightarrow$ distinct.\\
	  The valence increases by one.
	\end{description}
	Therefore, when the scheduler is centralized, the valence between any
	two consecutive configurations differs by at most one.
	\qed
\end{proof}

\begin{lemma}
	\label{lem:centralized:bivalent}
	Under a centralized scheduler, every execution fragment that starts in
	a multivalent configuration and ends in a univalent
	configuration contains a 1-bivalent configuration.
\end{lemma}

\begin{proof}
  By Lemma~\ref{lem:centralized:valence-one} and the centralized scheduler,
  we know that the valence between any two consecutive configurations differs
  by at most one.
        Since the execution fragment ends in a univalent configuration, the
        last multivalent configuration in the fragment must be bivalent. This
        configuration necessarily exists since the fragment starts in a
        multivalent configuration.

        Furthermore, since only one robot moves between any two configurations
        (centralized scheduler), the last bivalent configuration is 1-bivalent
        with the distinct robot doing the last move.
        \qed
\end{proof}

\subsection{Deterministic Gathering}
We begin by proving a theorem that strengthen the impossibility result of
Prencipe~\cite{Pre05} (Lemma~\ref{th:pre:nomultimposs}), as applied to the
problem of self-stabilizing gathering.
The theorem proves that the impossibility not only holds under a fair
scheduler, but also under a round-robin scheduler.

	\begin{figure*}
		\centering
		\begin{tabular}{*{3}{>{\centering}p{0.22\textwidth}<{}}@{}>{\centering}p{0.17\textwidth}<{}}
		%
			~ & ~ & ~ & ~
		\tabularnewline
			~ & ~ & ~ & ~
		\tabularnewline
		%
			\begin{tikzpicture}[remember picture,on grid,node distance=-0.25mm]
				\node[coordinate] (Lloc) at (0,0) {};
				\node[coordinate] (Rloc) at (1,0) {};
				\node[robot,       label=above:{$r_1$}] (r1) at (Lloc) {};
				\node[robot,       label=below:{$r_2$}] (r2) [below=of Rloc] {};
				\node[active_robot,label=above:{$r_{3\cdots n}$}] (r3) [above=of Rloc] {};
				\draw[move] (r3) -- (r1);
				\node[configuration] (Config1) at (0.5,0) {};
			\end{tikzpicture}
		&
			\begin{tikzpicture}[remember picture,on grid,node distance=-0.25mm]]
				\node[coordinate] (Lloc) at (0,0) {};
				\node[coordinate] (Rloc) at (1,0) {};
				\node[robot,       label=below:{$r_{3\cdots n}$}] (r3) [below=of Lloc] {};
				\node[robot,       label=below:{$r_2$}] (r2) at (Rloc) {};
				\node[active_robot,label=above:{$r_1$}] (r1) [above=of Lloc] {};
				\draw[move] (r1) -- (r2);
				\node[configuration] (Config2) at (0.5,0) {};
			\end{tikzpicture}
		&
			\begin{tikzpicture}[remember picture,on grid,node distance=-0.25mm]]
				\node[coordinate] (Lloc) at (0,0) {};
				\node[coordinate] (Rloc) at (1,0) {};
				\node[robot,       label=above:{$r_{3\cdots n}$}] (r3) at (Lloc) {};
				\node[robot,       label=below:{$r_1$}] (r1) [below=of Rloc] {};
				\node[active_robot,label=above:{$r_2$}] (r2) [above=of Rloc] {};
				\draw[move] (r2) -- (r3);
				\node[configuration] (Config3) at (0.5,0) {};
			\end{tikzpicture}
		&
			%
			%
			\begin{tikzpicture} [
				column sep=2mm,
				column 1/.style={anchor=base},
				column 2/.style={anchor=west,font=\itshape}
			]
			\matrix {
				\node [robot] {};
				& \node {robot};
				\\
				\node [active_robot] {};
				& \node {active robot};
				\\
				\node [coordinate] (x1) {};
				\draw [move] ($(x1)-(.25,0)$) -- ($(x1)+(.25,0)$);
				& \node {move};
				\\
				\node [coordinate] (x2) {};
				\draw [transition] ($(x2)-(.25,0)$) -- ($(x2)+(.25,0)$);
				& \node {transition} ;
				\\
				\node [configuration,rounded corners=2pt,minimum width=5mm,minimum height=4mm] {};
				& \node {configuration};
				\\
			};
			\end{tikzpicture}
		\tabularnewline
		%
		%
			\CONF[1] & \CONF[2] & \CONF[3] & ~
		\end{tabular}
		%
		\begin{tikzpicture}[remember picture,overlay,very thin,node distance=5mm]
			\node[coordinate] (a) [above=of Config3] {};
			\node[coordinate] (b) [above=of Config1] {};

			\draw[transition] (Config1) -- node[below]{$r_{3\cdots n}$ move} (Config2); 
			\draw[transition] (Config2) -- node[below]{$r_1$ moves} (Config3);
			\draw[transition] (Config3) -- (a) -- node[above]{$r_2$ moves} (b) -- (Config1);
		\end{tikzpicture}
		\caption{Proof of Theorem~\ref{th:detimpossibility-ff}:
			From a 1-bivalent configuration~\CONF[1] with robot~$r_1$ in a distinct
			location, an activation schedule $r_{3\cdots n},r_1,r_2,r_{3\cdots n},\ldots$
			generates a cycle of equivalent configurations.
			$r_{3\cdots n}$ represent the remaining robots $r_3$ to $r_n$.
			}
    \label{fig:impossibility-fault-free}
	\end{figure*}
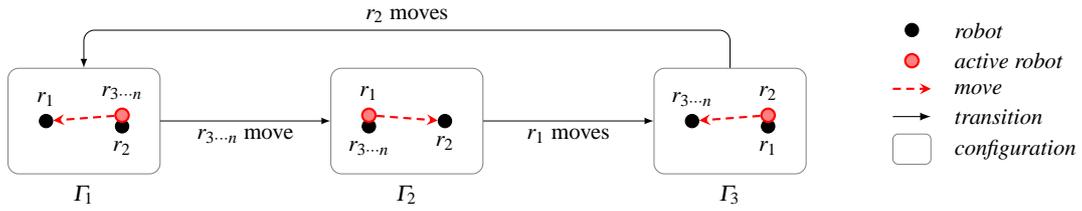
\begin{theorem}
  \label{th:detimpossibility-ff}
  Under a round-robin scheduler, there is no deterministic algorithm
  that solves self-stabilizing gathering for
  $n \geq 3$,
  without additional assumptions (e.g.,~multiplicity knowledge).
\end{theorem}
\begin{proof}
	Assume, by contradiction, that such an algorithm exists.
	Let \algo{A} be a deterministic algorithm that solves (distinct)
	gathering under a round-robin scheduler.
	
%
%
	
%
%
%

	Without loss of generality, let the reachable distance of the robots be so large that
	the robots can reach each others' location in a single step.

%

	Consider the initial configuration \CONF[1] as described below (see Fig~\ref{fig:impossibility-fault-free}).
	\CONF[1] is a 1-bivalent configuration such that all robots are at one location, except one
	robot, $r_1$, which is at a distinct location.
	
	Consider an execution~$e$ under algorithm~\algo{A} that starts in \CONF[1] and follows
	the round-robin activation schedule given by the sequence
	$\sigma = r_{3},\cdots, r_{n}, r_1, r_2$.
	The application of $\sigma$ to configuration \CONF[1] leads to a cycle of bivalent
	configurations, as illustrated in	Figure~\ref{fig:impossibility-fault-free}.
	
	Therefore, a univalent configuration is never reached in execution~$e$, which contradicts
	the fact that every execution under algorithm~\algo{A} satisfies the Convergence
	property.
	
	Thus, algorithm~\algo{A} does not exist.
	\qed
\end{proof}

The impossibility of Lemma~\ref{th:pre:nomultimposs} is for the
distinct gathering problem, namely when initial configurations are
restricted to distinct ones.

In order to extend that result, we must prove that the impossibility
also holds under the same conditions, namely, starting from distinct
configurations.
As stated earlier, an impossibility for distinct gathering implies an
impossibility for self-stabilizing gathering.

\medskip
We introduce an additional theorem below, which stricto senso extends the impossibility
of Lemma~\ref{th:pre:nomultimposs}.
We show that, under a $k\geq2$-bounded scheduler, \emph{every} multivalent
configuration (this includes every distinct configuration) can lead to a non-terminating
execution.
In other words, this proves the impossibility of distinct gathering.

\begin{theorem}
  \label{th:detimpossibility:k2}
  Under a centralized $k\geq2$-bounded scheduler, there is no deterministic
  algorithm that solves distinct gathering for $n \geq 3$,
  without additional assumptions (e.g.,~multiplicity knowledge).
\end{theorem}
\begin{proof}
	With no loss of generality, assume that the scheduler is 2-bounded since this is the
	most restrictive case for the adversary given the hypotheses of the lemma.
	
	Let the adversary select an arbitrary sequence $\sigma$ of the robots and activate them
	according to a round-robin policy over~$\sigma$.
	The scheduler is centralized, so Lemma~\ref{lem:centralized:bivalent} holds and thus,
	from any initial multivalent configuration, in particular any distinct configuration,
	if an algorithm exists, it must necessarily lead the system to a 1-bivalent configuration.
	Let \conf[x] be this 1-bivalent configuration and let $r$ denote the distinct robot.
	
	If robot~$r$ is not the next robot in $\sigma$, then continue the activations in a
	round-robin	fashion until another 1-bivalent configuration is reached, and repeat the argument.
	
	If robot~$r$ is the next robot in $\sigma$, then apply the following permutation.
	Let $r'$ be the robot in the one-before-last
	position in $\sigma$ and $r''$ the robot at the last position.
	The adversary updates $\sigma$ by letting $r$ swap positions with $r'$.
	This leads to the configuration~\CONF[1] depicted in Fig.~\ref{fig:impossibility-fault-free},
	where $r_1=r$, $r_2=r''$, and $r_3=r'$, and the cycle follows.
	
	The swap is valid under the 2-bounded scheduler because no further swap occurs, and no
	robot is activated more than twice between the last activation of $r$ before the swap and
	the first one after.	
	\qed
\end{proof}

We strongly believe that the impossibility under a round-robin scheduler applies not
only to self-stabilizing gathering, but also to distinct gathering.
We state the following conjecture under which the impossibility holds.
\begin{conjecture}
	\label{cj:criteria}\itshape
	Given a system with no additional assumptions (e.g, multiplicity knowledge).
	For every gathering algorithm under a round robin scheduler, 
	there exists an execution starting in a distinct configuration such that a gathered
	configuration is reached by activating exactly once, every robot except one.
\end{conjecture}

To substantiate why the claim might be true, let us consider some examples.

First, say that the criteria applied by the algorithm is to select the location of the
nearest robot. Then, one distinct configuration that meets the requirement of the lemma
is to place a first robot $r_1$, and then all other robots such that their distance to
$r_1$ follows a geometric progression. Activating each robots except $r_1$ in the order
they were placed let them gather at $r_1$.

Second, if the criteria is to select the farthest robot, then gathering is obtained from
placing the robots along a line and activating them from one extremity to the next.

Third, if the criteria is to select a robot near a centroid, then by placing robots along
the circumference of a circle centered at $r_1$ and an interleaved activation of the
robots, $r_1$ can still remain near the centroid until the system reaches a bivalent
configuration.

This is not exhaustive, and more complex criteria can be made to change depending on the
valence of the observation.

\begin{theorem}
  \label{th:detimpossibility-ff:distinct}
  Under the hypothesis that Conjecture~\ref{cj:criteria} holds,
  There is no deterministic algorithm that solves distinct gathering for
  $n \geq 3$ under a round-robin scheduler,
  without additional assumptions (e.g.,~multiplicity knowledge).
\end{theorem}
\begin{proof}
	Assume, by contradiction, that such an algorithm exists.
	Let \algo{A} be a deterministic algorithm that solves distinct
	gathering under a round-robin scheduler.

	Let the reachable distance of the robots be so large that the robots can reach
	each others' location in a single step.
	
	By assumption, Conjecture~\ref{cj:criteria} holds and there exists a
	configuration \CONF[0] such that, by activating every robot at most once,
	an execution~$e$ starting in \CONF[0] reaches a univalent configuration.
	Let us name the robots such that the successful activation sequence
	is $r_{3\cdots n}, r_2, r_1$.
	
	%
	%
	Since a robot moves only once, all robots must select the location of
	$r_1$ as their target, and the configuration after activating robots
	$r_{3\cdots n}$ is 1-bivalent with $r_2$ at the distinct location.
	
	Consider execution $e'$ with the same prefix, but where $r_1$ is
	activated before $r_2$. Without multiplicity, what $r_1$ observes in
	$e'$ is the same as what $r_2$ observes in $e$. Therefore, $r_1$ moves
	to the location of $r_2$, leading to configuration~\CONF[3] of
	Figure~\ref{fig:impossibility-fault-free}.
	And the rest follows.
	\qed
\end{proof}

%
%
Consider now the case when the system consists of two robots.
Suzuki and Yamashita~\cite{SY99} have proved that the deterministic gathering
of two oblivious robots is impossible under a fair scheduler (Lemma~\ref{th:pre:nodet2robots}).
The simple lemma below shows
that 2-gathering is however possible when the scheduler is centralized. 
\begin{lemma}
  \label{lem:sim:centr-2gathering}
  The 2-gathering problem can be solved deterministically under a
  centralized scheduler (fair or unfair).
\end{lemma}
\begin{proof}
  Let $r_1$ and $r_2$ be the two robots. Consider the simple algorithm
  which consists for one robot to move to the location of the other
  robot.
  Given that the scheduler is centralized, at each step only one of the
  two robots, say $r_1$, is active.
  
  If $r_2$ is reachable from $r_1$, then gathering is achieved in that
  step. 
  If $r_2$ is not reachable from $r_1$, then the distance between both
  robots decreases by $\delta_{r_1}$.

  Thus, by repeating the argument, we see that the distance between the
  robots decreases monotonically, until they become reachable and then
  gathering is achieved in the next activation.
  \qed
\end{proof}

Note that, in the above proof, it does not matter which robot is activated
in each round. In particular, even if the scheduler is unfair, it must activate
either one of the two robots.

\newcommand{\IND}{\hspace*{5mm}}

\subsection{Probabilistic Gathering}
We now look at the case of probabilistic algorithms in a fault-free environment.
In the following, we prove that, for the
case of two robots, there exists a probabilistic solution for
gathering in the SYm model, under any type of scheduler.
\begin{algorithm}
  \caption{Probabilistic gathering for robot $p$.}
  \label{alg:ff-gathering}
  \begin{small}
    \noindent
    \textbf{Actions}:\\
    \noindent
    \IND Observe($P$) :: $\mathit{true}$ $\longrightarrow$ \\
    \IND\IND\IND \textbf{with} probability $\alpha=\frac{1}{|P|}$ \textbf{do}\\
    \IND\IND\IND\IND select location $q \in P$ uniformly;\\
    \IND\IND\IND\IND move towards $q$;\\
    \IND\IND\IND \textbf{else}\\
    \IND\IND\IND\IND stay;
 	\end{small}
\end{algorithm}

Algorithm~\ref{alg:ff-gathering} describes the probabilistic strategy
of a robot. When a robot becomes active, it decides, with
probability~$\alpha$, whether it will actually compute a location and
move whereas, with probability~$1-\alpha$, the robot will remain
stationary.  The following lemma shows that
Algorithm~\ref{alg:ff-gathering} reaches a univalent configuration in
constant expected steps.

\begin{lemma}
  \label{lem:prob-2gathering}
  Algorithm~\ref{alg:ff-gathering} probabilistically solves
  gathering for $n=2$ under an unfair scheduler.
\end{lemma}
\begin{proof}
	Consider two robots $r_1$ and $r_2$, and an arbitrary initial configuration~\conf[0].
	If $r_1$ and $r_2$ are already gathered, the configuration is univalent and
	neither will move, regardless of activations and the probability~$\alpha$.
	
	Since there are two robots, every non-gathered configuration is bivalent, and thus
	$\alpha=\frac{1}{2}$.

	Assume that both robots have the same reachable distance $\delta$ (or, if they
	are different, define $\delta$ conservatively as their minimum).	
	
	Let us show how $r_1$ and $r_2$ reach a configuration in which they are mutually reachable,
	from one in which they are not. Let $D_0>\delta$ be the initial distance between $r_1$
	and $r_2$. At each successful move of either one of the robots, the distance between
	them is decreased by at least $\delta$ (by $2\delta$ if both move).
	Thus, it takes at most $x=\lceil\frac{D_0}{\delta}\rceil\!-\!1$ successful moves
	of either one robot, for them to be within reachable distance.
	Since the scheduler must activate at least one of the robots, the probability of a
	successful move at each trial is at least $\frac{1}{2}$. The number of failures until the
	$x^{th}$ success of a Bernoulli trial with success probability $\alpha$ is known to be
	a random variate that follows a negative binomial distribution $\mathit{NB}(x,\alpha)$.
	It follows that the expected number	of steps until both robots are within reachable
	distance~$\delta$ is at most
	\begin{align*}
		\E[\mbox{\footnotesize steps to reachable}]
		&\leq \lceil\frac{D_0}{\delta}\rceil\!-\!1
			 + \E\left[\mathit{NB}\left( \lceil\frac{D_0}{\delta}\rceil\!-\!1 ,\frac{1}{2} \right)\right]
		\\
		&\leq 2\left(\left\lceil\frac{D_0}{\delta}\right\rceil - 1\right)
		\\
	\end{align*}
%
	
	Let us now consider gathering from a configuration in which the two
	robots are reachable from each other.
	
	Consider some discrete time~$t$ when the two robots have distinct locations. If
	only one of the robots, say $r_1$, is activated by the scheduler, then
	there is a probability~$\alpha$ that $r_1$ moves, and thus both robots end up
	gathered in the next configuration (terminal). If both robots are
	activated at time~$t$, then they end up in a univalent configuration
	only if exactly one of them changes its position. This occurs with probability
	$2\alpha(1-\alpha)$.
	
	Consequently, the probability to reach gathering
	during at time~$t+1$ is at least
	$q=\min\left(\alpha,2\alpha\left(1-\alpha\right)\right)=\frac{1}{2}>0$, 
	regardless of the choice
	of the scheduler.
	The number of failures before first success is a random variate that follows
	a geometric distribution $G(q)$.
	This yields the expected number of steps until gathering as
	\[
	\E[\mbox{\footnotesize steps to gathering}] = 1+\E[G(q)] = 1+\frac{(1-q)}{q} = \frac{1}{q} = 2 
	\]
	Thus, gathering is achieved in at most $2\lceil\frac{D_0}{\delta}\rceil$ steps in
	expectation.
	\qed
\end{proof}




The next lemma extends the impossibility result proved in 
Theorem~\ref{th:detimpossibility-ff} to probabilistic 
algorithms under a fair centralized scheduler.

\begin{lemma}
  \label{lem:probimpossibility-ff}
  There is no probabilistic algorithm that solves gathering
  for $n \geq 3$, under a fair centralized
  scheduler without additional assumptions (e.g.,~multiplicity
  knowledge).
\end{lemma}
\begin{proof}
	A randomized algorithm can use randomization in two different ways.
	It can select random locations (case~A), or it can toss a coin before doing
	a move (case~B).
	
	With respect to the first case, the proofs of
	Theorem~\ref{th:detimpossibility-ff} and~\ref{th:detimpossibility:k2}
	still stand when destinations are based on random choices, except when
	the random choice of a robot is to select its current location.
	This is however equivalent to tossing a coin and stay still with some
	probability, which is in turn equivalent to the second case.
		
	Hence, we focus on the second case (case~B) and represent a randomized algorithm
	as one in which an active robot tosses a coin and, with some positive
	probability~$\alpha$, executes an action (and stays still otherwise). Note
	that, if the probability depends on the robot, $\alpha$ can be defined as the
	minimum.
	It must be positive because, since Theorem~\ref{th:detimpossibility:k2}
	shows that no algorithm exists based on deterministic choices, a robot
	cannot set the probability to zero based only on its observations.
	
	Consider an adversary that selects a robot~$r$ and activates $r$ until
	the coin toss is successful, and $r$ actually executes its action. Since
	$\alpha$ is positive, the activation is fair (albeit unbounded).
	By doing so, the adversary can actually ``\emph{derandomize}'' the algorithm
	with the remainder of the proof being the same as for Theorem~\ref{th:detimpossibility:k2}.
	\qed
\end{proof}

The key issue leading to the above impossibility is the freedom that
the scheduler has in selecting a robot~$r$ until its probabilistic
local computation allows $r$ to actually move. The scenario can
however no longer hold with systems in which the scheduler is
$k$-bounded.  That is, in systems where a robot cannot be activated
more than $k$~times before the activation of another robot.  In this
type of game, robots win against the scheduler and the system converges
to a gathered configuration.

\begin{theorem}
  \label{th:probpossibility}
  Algorithm~\ref{alg:ff-gathering} probabilistically solves gathering
  for $n\geq3$, under a fair
  bounded scheduler and without multiplicity knowledge.
\end{theorem}
\begin{proof}
	Let $k$ denote the bound of the scheduler. The scheduler being fair, there
	are at most $k(n-1)$ steps between any two consecutive activations of any
	robot. Let $\delta$ be the reachable distance of the robots (or their minimum
	if they are different).
	
	The probability $\alpha$ depends on the valence of the current configuration.
	However, in multivalent configurations, it is bounded as follows:
	$
		\frac{1}{n}=\alpha_{\min} \leq \alpha \leq \alpha_{\max}=\frac{1}{2}
	$.

	\medskip
	For clarity, the proof has two parts. First, we show that, from an arbitrary configuration,
	the system reaches a configuration in which all robots are within reachable distance
	from each other. Second, we show that, with high probability, in a configuration
	where robots are reachable from each other, the valence of successive configurations
	decreases until gathering is reached.

	\paragraph{Theorem~\ref{th:probpossibility}; Part~1: from arbitrary to reachable.}
	Consider the smallest enclosing circle~\SEC[t] defined by the robots'
	locations in a configuration~\conf[t]. By definition of the smallest enclosing circle,
	and because a circle is convex,	all locations and all segments between them are
	either inside the circle or on its circumference. By Algorithm~\ref{alg:ff-gathering},
	a robot~$r$ selects a target $r'$ among the robots' locations, $r$ can move only to
	$r'$ or to some point in the segment between them. Thus, $r$ in \conf[t+1] is
	necessary enclosed by \SEC[t]. Thus, $\SEC[t+1] \subseteq \SEC[t]$. In other words,
	the smallest enclosing circle is non-increasing.

	To show the convergence, we now show that, with some positive probability~$p$, the
	diameter of the smallest enclosing circle decreases by at least $\delta$, which is
	a constant positive value.
		
	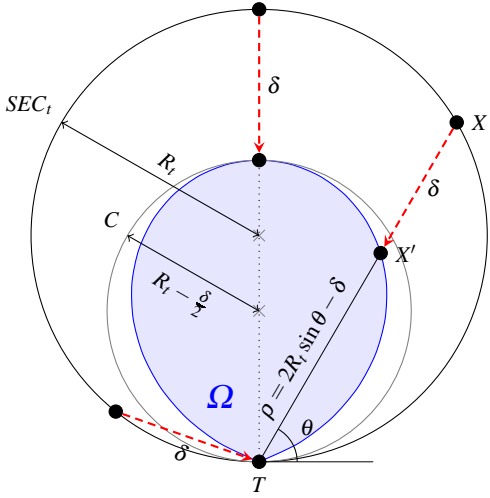
\begin{figure}
		\centering
		\begin{tikzpicture}
			\def\ddelta{2}
			\def\dradius{3}
			\def\thetaX{60}
			\def\diameter{(2*\dradius)}
			\def\asin{rad(asin(\ddelta/\diameter))}

			\draw [blue,fill=blue!10,smooth,domain=\asin:3.141-\asin,variable=\theta]
				plot (
					{(\diameter*sin(\theta r)-\ddelta)*cos(\theta r)},
					{(\diameter*sin(\theta r)-\ddelta)*sin(\theta r)}
				);
			
			\coordinate (Target) at (0,0) {};
			\coordinate (Top)    at (0,2*\dradius) {};
			\coordinate (Top')   at ($(Top)-(0,\ddelta)$) {};
			\coordinate (SEC-center) at (0,\dradius) {};
			\coordinate (C-center)   at ($(0,\dradius)-(0,\ddelta/2)$) {};
			\coordinate (Max)    at ($(Target)+({180-asin(\ddelta/\diameter)}:\ddelta)$) {};
			
			\node[blue,scale=1.5] at (120:1) {$\Omega$};
						
			\node[coordinate] (Xp) at ($({\diameter*sin(\thetaX)*cos(\thetaX)},{\diameter*sin(\thetaX)*sin(\thetaX)})$) {};
			\node[coordinate] (X'p) at ($({(\diameter*sin(\thetaX)-\ddelta)*cos(\thetaX)},{(\diameter*sin(\thetaX)-\ddelta)*sin(\thetaX)})$) {};
			
			\draw            (SEC-center) circle (\dradius);
			\draw[thin,gray] (C-center)   circle (\dradius-\ddelta/2);
			
			\draw[dotted] (Target) -- (Top');
			
			\draw[very thin] (X'p) -- node[above,sloped] {$\rho = 2R_t\sin\theta-\delta$} (Target);
			\draw[very thin] (Target) -- +(\dradius/2,0);
			
			\path (Top) -- node[right] {$\delta$} (Top');
			\path (Xp)  -- node[right] {$\delta$} (X'p);

			\node[robot,label=below:{\point{T}}]  (T) at (Target) {};

			\node[robot] (TopR)  at (Top) {};
			\node[robot] (TopR') at (Top') {};
			\draw[move]  (TopR) -- (TopR');
	
			\node[robot,label=right:{\point{X}}]  (X) at (Xp)   {};
			\node[robot,label=right:{\point{X'}}] (X') at (X'p) {};
			\draw[move] (X) -- (X');
			
			\node[robot] (MaxR) at (Max) {};
			\draw[move]  (MaxR) -- (T);
			\path (MaxR) -- node[below,sloped] {$\delta$} (T);
			
			\draw (5mm,0) arc (0:\thetaX/2:5mm)	node[anchor=south west] {$\theta$} arc (\thetaX/2:\thetaX:5mm);

			\node[point] at (SEC-center) {};
			\node[point] at (C-center) {};
			\draw[thin,<->] (SEC-center)
				 -- node[above,sloped] {$R_t$}
				 		+(150:\dradius)
						node[anchor=south east] {\SEC[t]};
			\draw[thin,<->] (C-center)
				 -- node[below,sloped] {$R_t-\frac{\delta}{2}$}
				 		+(150:{\dradius-\ddelta/2})
						node [anchor=south east] {$C$};
			
		\end{tikzpicture}
		\caption{Proof of Theorem~\ref{th:probpossibility}; Part~1: Strict decrease of
		the smallest enclosing circle \SEC[t] (of radius~$R_t$) occurs with positive
		probability.
		After all robots
		select \point{T} as a target and move once (dashed lines), they are all contained inside area $\Omega$,
		which is itself contained within a circle~$C$ of radius $R_t-\frac{\delta}{2}$.
		Positions are expressed in polar coordinates centered at \point{T}.}
		\label{fig:SEC:decrease}
	\end{figure}
	
	Let \conf[t] be a configuration and \SEC[t] the smallest enclosing circle in \conf[t].
	Let \point{T} be the position of a robot located on the boundary of \SEC[t].

	Consider the situation in which all other robots take \point{T} as their target for a
	successful move and let	\conf[x] be the resulting configuration
	(Fig.~\ref{fig:SEC:decrease}).
	We show that the
	smallest enclosing circle \SEC[x] in configuration \conf[x] is smaller than \SEC[t]
	in diameter by at least $\delta$.
	
	To characterize the movement of the robots, we consider polar coordinates
	$(\theta,\rho)$ centered at \point{T}. The smallest enclosing circle \SEC[t] is given by
	\[
		\theta \in \left[0 ; \pi\right]
		\\
		\rho = 2R_t\sin\theta
	\]
	Let $\Omega$ be the area that the robots will reach after moving toward $T$ of
	a distance at least $\delta$ (Fig.~\ref{fig:SEC:decrease}). This area can be
	characterized as follows%
	\footnote{
		Note that $\Omega$ is not a circle; it is best described as the inner loop of a limaçon of Pascal.
	}
	\[
		\theta \in \left[ \arcsin\frac{\delta}{2} ; \pi\!-\!\arcsin\frac{\delta}{2} \right]
		\\
		\rho = 2R_t\sin\theta - \delta
	\]
	Let $C$ be a circle with diameter $2R_t-\delta$ and anchored at \point{T}.
	\[
		\theta \in \left[ 0 ; \pi \right]
		\\
		\rho = (2R_t-\delta)\sin\theta
	\]
	Let us show that $\Omega$ is contained within $C$. When
	\[
		\theta \in \left[ \arcsin\frac{\delta}{2} ; \pi-\arcsin\frac{\delta}{2} \right]
	\]
	this holds if the following inequality always holds
	\begin{align*}
		(2R_t-\delta)\sin\theta &\geq 2R_t\sin\theta - \delta
		\\
		-\delta\sin\theta &\geq - \delta
		\\
		\sin\theta &\leq 1
	\end{align*}
	which is always true. Since all robots are contained within $C$, it follows
	that \SEC[x] is contained within $C$, and thus its diameter is at most
	$2R_t-\delta$.
	
	Notice that selecting \point{T} on the boundary is a worst case	for convergence; the
	best case occurs when \point{T} is located near the center of \SEC[t] and the decrease
	in diameter then becomes $2\delta$.

	Let us now show that, with positive probability~$p$, \SEC decreases by at least
	$\delta$ in diameter in a constant number of activation steps.
	
	Note that we do not need to calculate the probability $p$ accurately. It is
	sufficient to show that $p$ has a positive lower bound. For this reason, we
	do not need to consider all cases.
		
	Consider an	execution fragment~\fragment{t}{t+K}
	of $K=nk$ successive configurations	starting in \conf[t].
	The scheduler is fair $k$-bounded so, regardless of its choices, every robot
	is activated at least once and at most $nk$ times in fragment \fragment{t}{t+K}.
	
	We need to calculate the probability that (1)~every robot except one (at \point{T})
	makes exactly one successful move toward \point{T} at first trial and takes no
	further move in the $K-1$ remaining steps, and (2)~one robot (at \point{T}) makes
	no successful move in $K$ steps. 
	
	For one of the robots (except one at \point{T}), the probability that it makes a
	successful move toward \point{T} at first trial is
	\[
		\PP[\mbox{1 robot moves to \point{T} at first trial}]
	  	\geq \alpha_{\min} \frac{1}{n}
			= \frac{1}{n^2}
	\]
	Assuming a worst case situation when the scheduler activates the robot every time after the
	move, the probability for that robot to take no successful move in $K-1$ steps is
	\[
		\PP[\mbox{1 robot stays for $K-1$ steps}]
	  	\geq (1-\alpha_{\max}) ^ {K-1} = \frac{1}{2^{K-1}}
	\]
	For the $n-1$ robots, we combine and obtain
	\[
		\PP[\mbox{$n\!-\!1$ robots move once then stay}]
			\geq \frac{1}{ \left(n^2 2^{K-1}\right) ^ {n-1} }
	\]
	Assuming again a worst case scheduling decision, the probability that the
	robot located at \point{T} takes no move in all $K$ steps is
  \[
  	\PP[\mbox{$T$ stays for $K$ steps}]
    	\geq (1-\alpha_{\max}) ^ K = \frac{1}{2^K}
	\]
	and we combine all this to obtain the probability that one robot (at \point{T}) takes
	no move and that all other robots move exactly once with \point{T} as their target
	\[
  	\PP[\cdots]
			\geq
			\frac{1}{ 2^{nK-n+1} n^{2(n-1)} }
	\]
	That probability $\PP[\cdots]$ is a strictly positive constant because $K$ and $n$
	are both positive constants.
	To sum up, the probability that all robots are contained in a circle with
	diameter decreased by constant $\delta>0$ in constant $K=nk$ steps is
	\emph{at least} $\PP[\cdots]$.

	An upper bound on the expected number of steps need for all robots to be
	reachable is now easily obtained from a negative binomial distribution, following
	the same method as in Lemma~\ref{lem:prob-2gathering}. With $D_0$ being the diameter
	of \SEC in the initial configuration, the number of successful progress necessary is
	$x=\lceil\frac{D_0}{\delta}\rceil\!-\!1$, and we obtain, after much simplification,
	\[
		\E[\mbox{steps to reachable}]
		\leq
		K\frac{x}{\PP[\cdots]}
		=
		 k
		 \frac{\lceil\frac{D_0}{\delta}\rceil\!-\!1}
		 {
		 	2 ^ {kn^2\!-\!n\!+\!1}
			n ^ {2n\!-\!3}
		 }
	\]
	which is constant since $D_0$, $k$, and $n$ are all constant.
	This completes the proof of the first part.

	\paragraph{Theorem~\ref{th:probpossibility}; Part~2: From reachable to gathering.}
	Starting from a configuration in which all robots are mutually reachable, we show that
	we reach gathering with high probability.
	
	First, since the smallest enclosing circle is non increasing, once a reachable
	configuration has been achieved, all subsequent configurations are reachable.
	
	Taking an execution fragment of $nk$ steps starting in a reachable
	configuration, there is a positive probability that
	(1)~every robot except one (say \point{T'}) makes exactly one successful move toward \point{T'}
	at first trial and takes no	further move in the $nk-1$ remaining steps, and
	(2)~one robot (at \point{T'}) makes	no successful move in $nk$ steps. 

	This is the same probability as one fragment of $nk$ steps considered in the first part,
	so one additional successful fragment will reach gathering. This yields,
	\[
		\E[\mbox{steps from arbitrary to gathering}]
		\leq
			k
		 \frac{\lceil\frac{D_0}{\delta}\rceil}
		 {
		 	2 ^ {kn^2\!-\!n\!+\!1}
			n ^ {2n\!-\!3}
		 }
	\]
	Thus, gathering	is achieved in constant expected steps.	
	\qed
\end{proof}

%% file: ftss.tex

\section{Crash-Tolerant Self-Stabilizing Gathering}
\label{sec:ft-gathering}
We now extend the study on the feasibility of gathering to fault-prone
environments. In this section, we consider the family of $(n,f)$-crash
models, where $n$ is the total number of robots and up to $f<n$ of them
are faulty and may possibly crash.

Recall the two definitions for self-stabilizing gathering, namely,
strong and weak. Let us first state a simple impossibility about strong
gathering in systems with multiple faults.
\begin{lemma}
	\label{lem:strong:2f:imp}
	No algorithm can possibly solve strong gathering in a crash-prone system with $f>1$.
\end{lemma}
\begin{proof}
	Consider any multivalent initial configuration ($f>1$ thus $n>1$).
	Take any two robots with distinct locations and let them crash before they move.
	They cannot move, so they will never share the same location, and hence strong
	gathering is never achieved.
	\qed
\end{proof}
Notice that the above lemma holds regardless of the scheduler or additional
assumptions of any kind, and obviously applies to both deterministic and
probabilistic algorithms.

This leaves for study
the case of strong gathering in the face of a single faulty robot
(in Sect.~\ref{sec:crash:f=1-strong}),
and the case of weak gathering with multiple faulty robots
(in Sect.~\ref{sec:crash:f>1-weak}).

\subsection{Single Crash ($f=1$); Strong Gathering}
\label{sec:crash:f=1-strong}
We investigate the feasibility of strong gathering in the presence of a single
faulty robot. We express the impossibility lemmas to cover one or several
robots and simply refer to Lemma~\ref{lem:strong:2f:imp} for multiple robots
so that the proof can focus on the case of a single crash.
\begin{lemma}
	\label{lem:imp-ff-gathering}
  \label{(3,1)gathering}
  In an $(n,f)$-crash system with $n\geq3$ and $f\geq 1$,
  strong gathering is deterministically
  impossible under a round-robin scheduler, even with multiplicity knowledge.
\end{lemma}
\begin{proof}
	The case when $f>1$ is covered by Lemma~\ref{lem:strong:2f:imp},
	which leaves the case when $f=1$.

	By contradiction, assume that an algorithm \algo{A} solves gathering
	deterministically in an $(n,1)$-crash system under a round-robin scheduler.
	
	From Theorem~\ref{th:detimpossibility-ff}, \algo{A} must rely on multiplicity knowledge.
	The impossibility of Theorem~\ref{th:detimpossibility-ff} indeed applies here because
	a crash-free execution is valid in an $(n,1)$-crash system. Assuming that \algo{A}
	could solve gathering without multiplicity knowledge would imply that it can solve it
	in a fault-free execution; a contradiction.	

	A round-robin scheduler is centralized by definition, so
	Lemma~\ref{lem:centralized:bivalent} applies and the system must
	necessarily reach a 1-bivalent configuration before it achieves gathering.
	Let \conf[r] be such a configuration and let $r$ be the distinct robot.
	Consider also an arbitrary robot in the other location and call it $r'$.
	
	It is easy to see that $r'$ cannot chose to move to $r$ in a 1-bivalent
	configuration, or else,
	an adversary can lead a fault-free execution to the	same cyclic
	execution described in the proof of Theorem~\ref{th:detimpossibility-ff}
	(Fig.~\ref{fig:impossibility-fault-free}).
	
	Now, consider the case when $r$ crashes in \conf[r].
	The algorithm is deterministic and robots are oblivious, so $r'$ cannot
	distinguish that configuration from the fault-free one.
	The same decision must hence be applied and thus $r'$ and $r$ can never
	share the same location; a contradiction.
	\qed
\end{proof}

We now prove a similar impossibility for probabilistic algorithms, but
this time, under a fair centralized scheduler.
\begin{lemma}
  \label{lem:prob:(n,f)strong:imp}
  \label{lemma:(n,1)ip}
  \label{(3,1)pstrong-gathering}
  In an $(n,f)$-crash system with $n \geq 3$ and $f\geq1$,
  there is no probabilistic algorithm that solves strong gathering
  under a fair centralized scheduler,
  even with multiplicity knowledge.
\end{lemma}
\begin{proof}
	The case when $f>1$ is covered by Lemma~\ref{lem:strong:2f:imp},
	which leaves the case when $f=1$.

	By contradiction, assume that an algorithm \algo{A} solves gathering
	probabilistically in an $(n,1)$-crash system under a fair centralized scheduler.
	
	Lemma~\ref{lem:centralized:bivalent} applies since the scheduler is centralized,
	and any gathering execution must reach a 1-bivalent configuration~\conf[r].	
	Let $r$ be the distinct robot in \conf[r] and \point{G} the location of the other robots.
	
	We can now construct an adversary that prevents gathering.
	First, let $r$ be the faulty robot and let it crash in \conf[r] before it moves.
	Second, let the adversary activate the correct robots in turn. Each time a robot
	$r'$ moves to $r$, it is activated repeatedly until it moves back to \point{G}.
	The move to \point{G} must be possible or else \point{G} could not form in the first place
	(recall that robots are oblivious, anonymous, and disoriented).
	This activation is compatible with the fair centralized scheduler because every
	correct robot is activated infinitely often (fair) and in mutual exclusion
	(centralized). This leads to an infinite execution that holds no univalent
	configuration.
	Thus, \algo{A} violates the Convergence property of gathering; a contradiction.
	\qed
\end{proof}

We present now a simple lemma for a probabilistic algorithm in a system with two robots.
\begin{lemma}
	\label{lem:prob:(2,1)strong}
	In a $(2,1)$-crash system, Algorithm~\ref{alg:ff-gathering} solves
	the strong gathering problem probabilistically under an unfair scheduler.
\end{lemma}
\begin{proof}
	In a fault-free execution, the proof of	Lemma~\ref{lem:sim:centr-2gathering}
	applies as it stands.
	In an execution with one crash, gathering	is achieved through repeated
	activations of the correct robot. By the non-triviality condition of the
	unfair scheduler, it must activate the correct robot infinitely often.
	\qed
\end{proof}
%
%

Based on the same proof, we obtain a lemma for the deterministic gathering of two robots under a centralized scheduler.
\begin{lemma}
	\label{lem:det:(2,1)strong}
	In a $(2,1)$-crash system, Algorithm~\ref{alg:ff-gathering} solves
	strong gathering probabilistically under an unfair centralized scheduler.
\end{lemma}

\medskip
In a system with more than two robots, strong gathering can still be achieved
with a probabilistic algorithm, but requires the scheduler to be fair bounded.
The algorithm does not need to rely on multiplicity knowledge.
\begin{lemma}
  \label{lemma:coralg-ff-gathering}
  In an $(n,1)$-crash system with $n\geq3$ and under a fair bounded scheduler,
  Algorithm~\ref{alg:ff-gathering} solves
  strong gathering probabilistically.
\end{lemma}
\begin{proof}
	The proof is identical to that of Theorem~\ref{th:probpossibility}, where
	the target location $T$ of the first part is chosen as the location of the
	faulty robot (crashed or not), and the second part requires no adaptation.
	\qed
\end{proof}

\subsection{Multiple Crashes ($f\geq2$); Weak Gathering}
\label{sec:crash:f>1-weak}
We now extend the study to the case of weak gathering in the presence
of multiple faulty robots.

We begin by proving the impossibility of deterministic and
probabilistic weak gathering under a round-robin scheduler
without additional assumptions.
\begin{lemma}
  \label{lemma:impweak}
  In a $(n,f)$-crash system with $n \geq 3$ and $f \geq 2$, there
  is neither a probabilistic nor a deterministic algorithm that solves
  weak gathering under a round-robin scheduler, 
  without additional assumptions.
\end{lemma}
\begin{proof}
	By contradiction, assume that such an algorithm exists and call it $\mathcal{A}$.

	Consider a fault-free execution~$e$.
	The scheduler being centralized (implied by round-robin)
	Lemma~\ref{lem:centralized:bivalent} holds and every execution under
	algorithm~\algo{A} reaches a 1-bivalent configuration~\conf[b].
	Let $r_1$ be the distinct robot, and $\{r_2,\cdots, r_n\}$ the robots
	in the other location.
	
	Consider an execution~$e'$ which differs from $e$ in that $r_1$ and $r_2$ both crash in
	configuration \conf[b],
	leading to the 1-bivalent configuration~$\conf[b]'$.
	In the absence of multiplicity, bivalent configurations are
	undistinguishable for the robots.
	
	So, some robot $r_x$ in $\{r_{3},\cdots, r_n\}$ is correct and has a positive
	probability of moving to the other location.
	To see this, consider that, otherwise, an adversary can generate an execution
	that is unable to transit from a 1-bivalent configuration to a univalent
	configuration in the fault-free case; a contradiction.

	Take the more general case and assume that \algo{A} is probabilistic.
	Assuming $\conf[b]'$ does not change, the scheduler ensures that $r_x$ is activated
	infinitely often.
	It follows that, with high probability, $r_x$ moves, violating the Closure property;
	a contradiction.
	\qed
\end{proof}

An immediate consequence of the previous lemma is the necessity of an
additional assumption, such as multiplicity knowledge, even for
probabilistic solutions and even under round-robin or bounded schedulers.
Accordingly, we now consider systems in which robots are aware of multiplicity.

\subsubsection{Deterministic weak gathering with multiple crashes}

Algorithm~\ref{alg:ft-gathering} is a deterministic algorithm that relies on
multiplicity detection. Roughly, when a robot $r$ becomes active, it considers
the castles in the current configuration. If there are castles to which $r$
does not belong, then it moves to the nearest one, say \castle{q}, with ties broken
arbitrarily.

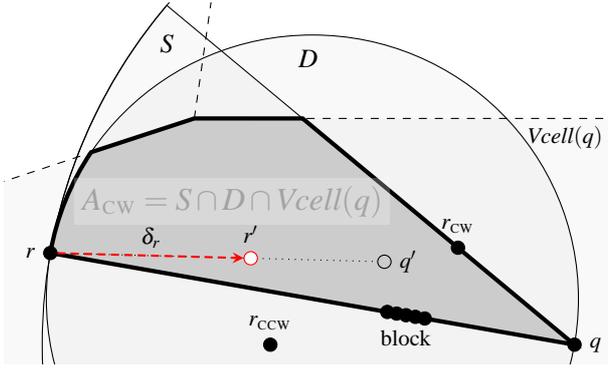
\begin{figure}
	\centering
	\begin{tikzpicture}
		\def\radius{7}
		\def\angleP{170}
		\def\angleS{140}
		
		\clip (-7.5,-0.25) rectangle (0.5,4.53);
		
		\coordinate (q) at (0,0);
		\coordinate (p) at (\angleP:\radius);
		\coordinate (rcw)  at (\angleS:2);
		\coordinate (rccw) at (180:4);
		\coordinate (blockS) at (\angleP:2);
		\coordinate (blockE) at (\angleP:2.5);
		
		\node[robot,label=left:{$r$}] (np) at (p) {};
		\node[robot,label=right:{\castle{q}}] (nq) at (q) {};
		\draw[thick] (p) -- (q);
		\node[robot,label=above:{$r_{\textsc{cw}}$}]  (nrcw)  at (rcw) {};
		\node[robot,label=above:{$r_{\textsc{ccw}}$}] (nrccw) at (rccw) {};
		\node[robot,label=below:{block}] (nblock) at ($(blockS)!0.5!(blockE)$) {};
		\foreach \x in {0,0.25,...,0.49} {
			\node[robot] at ($(blockS)!\x!(blockE)$) {};
			\node[robot] at ($(blockE)!\x!(blockS)$) {};
		}
		
		\node [draw] at (q) [circle through={(p)}] {};
		\draw[name path=S] (q) -- (p) arc (\angleP:\angleS:\radius) -- (q);
		\begin{pgfonlayer}{background}
			\fill[gray,opacity=0.05] (q) -- (p) arc (\angleP:\angleS:\radius) -- (q);
		\end{pgfonlayer}
		\node at ($(\angleS:\radius)+(0,-.5)$) {\normalsize$S$};
		
		
		\path ($(p)!0.5!(q)$) circle (3.5);
		\node [draw] at ($(p)!0.5!(q)$) [circle through={(q)}] {};
		\begin{pgfonlayer}{background}
			\clip (-7.5,-0.25) rectangle (0.5,4.53);
			\node [fill=gray,opacity=0.05] at ($(p)!0.5!(q)$) [circle through={(q)}] {};
		\end{pgfonlayer}
		\node at (-3.5,3.8) {\normalsize$D$};
		
		\draw[name path=Vcell,dashed] (-8,2) -- (-5,3) -- node[below,near end]{$\Vcell{q}$} (1.5,3);
		\draw[dashed] (-5,3) -- (-4,10);
		\begin{pgfonlayer}{background}
			\clip (-7.5,-0.25) rectangle (0.5,4.53);
			\fill[gray,opacity=0.05] (-8,2) -- (-5,3) -- (2,3) -- (2,-1) -- (-8,-1) -- cycle;
		\end{pgfonlayer}
		
		\path [name intersections={of=S and Vcell,by={dummy,C}}];
		\draw[ultra thick] (q) -- (p) arc (\angleP:146.25:3.5) -- (-5,3) -- (C) -- cycle;
		\begin{pgfonlayer}{background}
			\fill[gray!40] (q) -- (p) arc (\angleP:146.25:3.5) -- (-5,3) -- (C) -- cycle;
		\end{pgfonlayer}
		\node[fill=white,fill opacity=0.25,text opacity=1] at (-4.5,1.9) {\large $A_{\textsc{cw}} = S \cap D \cap \Vcell{q}$};
		
		\node[robot,fill=none,label=right:{\point{q'}}] (dest) at (-2.5,1.1) {};
		\draw[dotted] (np) -- (dest);
		\node[robot,red,fill=white,label=above:{$r'$}] (p') at ($(p)!0.6!(dest)$) {};
		\draw[move] (np) -- node[black,sloped,above] {$\delta_{r}$} (p');
		
	\end{tikzpicture}
	\caption{
		\emph{Side Move:}
		Robot $r$ selects maximal multiplicity point \castle{q} as its target.
		Some robots block the way between $r$ and \castle{q}.
		Robot $r_{\textsc{cw}}$ is the next robot clockwise (arbitrary choice) with respect
		to \castle{q}.
		Robot $r$ selects a point inside area $A_{\textsc{cw}}$, called \point{q'}, and moves
		toward it.
	}
	\label{fig:side-move}
\end{figure}

\begin{algorithm}
  \caption{Deterministic fault-tolerant weak gathering for robot $r$ at location \point{p}}
  \label{alg:ft-gathering}
  \label{alg:ft-det-gathering}
  \begin{small}
    \textbf{Functions}:\\
    \IND $\mulmax{\multiset{\Pi}}::$  the maximal multiplicity.\\
    \IND $\MaxMult{\multiset{\Pi}}::$ the set of elements with multiplicity \mulmax{\multiset{\Pi}}.\\
    \IND $\rbetween{\point{p}}{\point{q}}::$ the number of robots on segment $\overline{\point{p}\point{q}}$.\\

    \noindent
    \textbf{Actions}:\\
    \noindent
    \IND Observe(\multiset{\Pi}) :: $\mathit{true}$  $\longrightarrow$ \\
    \IND\IND\IND \textbf{if} $\exists \point{q}\not=\point{p} : \point{q} \in \MaxMult{\multiset{\Pi}}$ \textbf{then}\\
    \IND\IND\IND\IND select nearest $\point{q} \in \MaxMult{\multiset{\Pi}}\setminus \{\point{p}\}$;\\
		\IND\IND\IND\IND \textbf{if} $\rbetween{\point{p}}{\point{q}} < 2\mulmax{\multiset{\Pi}}$ \textbf{then}\\
    \IND\IND\IND\IND\IND \textsl{// Straight Move}\\
    \IND\IND\IND\IND\IND move toward \point{q};\\
    \IND\IND\IND\IND \textbf{else}\\
    \IND\IND\IND\IND\IND \textsl{// Side Move (see Fig.~\ref{fig:side-move})}\\
    \IND\IND\IND\IND\IND let $C$ := circle with center \point{q} and radius $\point{p}\point{q}$;\\
    \IND\IND\IND\IND\IND let $S$ := largest sector in $C$ from \point{p} with no robot;\\
    \IND\IND\IND\IND\IND let $D$ := disc with diameter $\point{p}\point{q}$;\\
    \IND\IND\IND\IND\IND select some \point{q'} inside $A = S \cap D \cap \Vcell{\point{q}}$;\\
    \IND\IND\IND\IND\IND move toward \point{q'};\\
    \IND\IND\IND \textbf{else}\\
    \IND\IND\IND\IND stay;
  \end{small} 
\end{algorithm}

When several robots occupy the space between $r$ and \castle{q}, moving over these
robots brings the risk of a cyclic behavior that can never converge (see details
in the appendix). To prevent this, $r$ is required to perform a \emph{side move}
(see Fig.~\ref{fig:side-move}), in which $r$ selects a destination within a zone
$A_{\textsc{cw}}$ (resp. $A_{\textsc{ccw}}$) constructed as the intersection of three
areas (boundaries excluded).
\begin{itemize}
\item \Vcell{\castle{q}}: cell of castle \castle{q} in the Voronoi diagram built from the set of castles.

\item $D$: the disc with segment $\overline{\point{r}\point{q}}$ as diameter.

\item $S$: the largest sector clockwise (resp. counter-clockwise) centered at $q$
	and starting from $r$ that contains no robot.
\end{itemize}
By moving within the zone, this ensures that
(1)~the distance between $r$ and \castle{q} decreases ($r$ moves within $D$),
(2)~\castle{q} remains the nearest castle from $r$ ($r$ moves within \Vcell{\castle{q}}),
and (3)~there are no robots between $r$ and \castle{q} ($r$ moves within $S$).

\medskip
Before proving the correctness of Algorithm~\ref{alg:ft-gathering}, we establish two
of its important properties.


The first property is an observation that the maximal multiplicity of
configurations throughout executions of Algorithm~\ref{alg:ft-gathering}
is non-decreasing. It holds for any centralized scheduler.

\begin{proposition}
	\label{prop:ft-det:maxmul}
	The maximal multiplicity of configurations is non-decreasing
	over any execution of Algorithm~\ref{alg:ft-gathering},
	under an unfair centralized scheduler.
\end{proposition}
\begin{proof}
	In configurations with a single point of maximal multiplicity,
	the condition of the test in Algorithm~\ref{alg:ft-gathering} evaluates
	to false for any robot $r$ that is on the point of maximal multiplicity,
	and thus, $r$ does not move and the multiplicity does not change.
	
	When there are several points of maximal multiplicity, they can be
	destroyed (one of its robot leaving the location, its multiplicity
	decreases) only one at a time because the scheduler is centralized.
	\qed
\end{proof}

Another important property, which holds for any fair scheduler (i.e., not necessarily
a centralized one), states that, in distinct configurations, the minimum distance between
two robots is non-increasing.

\begin{proposition}
	\label{prop:ft-det:D:nonincreasing}
	Consider an execution of Algorithm~\ref{alg:ft-gathering} under a fair scheduler.
	Let $D(\conf)$ be a function defined as the shortest distance between a robot and its
	nearest neighbor in configuration~\conf.
	Then, $D(\conf)$ is non-increasing.
\end{proposition}
\begin{proof}
	Assume by contradiction that
	there is a configuration \conf[t] such that $D(\conf[t])<D(\conf[t+1])$. Let $r$ and $r'$ be
	two	robots with distance $D(\conf[t])$ from each other in \conf[t].
	If neither $r$ nor $r'$ move at time~$t$, then $D(\conf[t])=D(\conf[t+1])$. So assume that at
	least one	of them moves at time $t$.
	In \conf[t+1], the distance from $r$ to $r'$ must have increased to be at least $D(\conf[t+1])$
	so, one of the robots must have moved away from the other, say $r'$ moved away from $r$.
	This means that $r'$ had a neighbor $r''\not=r$ such that	the distance from $r'$ to $r''$ was:
	(1)~at most	$D(\conf[t])$, or else $r$ (not $r''$) would be the nearest neighbor of $r'$
	    and $r'$ must have moved toward $r$, and
	(2)~at least $D(\conf[t+1])+\delta$ since, after moving, the nearest distance from a correct
	robot to its nearest neighbor is at least $D(\conf[t+1])$.
	It follows that $D(\conf[t]) \geq D(\conf[t+1])+\delta$.
	A contradiction with $D(\conf[t])<D(\conf[t+1])$.
	
	Hence, $D(\conf)$ is non-increasing.
	\qed
\end{proof}

%

We can now show that a distinct configuration eventually leads to a configuration
that contains a castle.

\begin{proposition}
	\label{prop:ft-det:distinct2reachable}
	In an $(n,f)$-crash system, where $n\!>\!f$, a fair centralized scheduler and
	multiplicity detection,
	let $e$ be an execution (or execution suffix) when robots move to the nearest robot.
	Starting from any distinct configuration, then $e$ contains a configuration with
	maximal multiplicity larger than one.
\end{proposition}
\begin{proof}
	We show that, starting from any distinct configuration, a location with
	multiplicity~2 is eventually formed.

	Consider again the function $D(\conf)$ defined as the shortest distance from a robot to its
	nearest neighbor. We know already from Proposition~\ref{prop:ft-det:D:nonincreasing} that
	$D(\conf)$ is non-increasing.
	
	We now show that there is a configuration such that $D(\conf)$ decreases strictly.
	Consider some distinct configuration \conf[t], and let $r$ be a correct
	robot with distance $D=D(\conf[t])$ to its nearest neighbor $r'$ in \conf[t].
	Then, there must be a configuration~\conf[t'] ($t'>t$) during which one of the following
	situation occurs:
	\begin{enumerate}
	\item $r'$ moves away from $r$.
		This means that there is a robot $r''$, originally at distance $D$ or less from $r'$,
		toward which $r'$ moves. The distance from $r'$ to $r''$ is at most $D(\conf[t])-\delta$,
		so $D(\conf[t'])\leq D-\delta$.
				
	\item $r$ moves and $r'$ is still its nearest neighbor.
		$r$ moves toward $r'$ by distance $\delta$, so $D(\conf[t'])\leq D-\delta$. 
		
	\item $r$ has a robot $r''$ as nearest neighbor.
		There are three cases:
		\begin{enumerate}
		\item either $r$ or $r'$ have moved, then we have already encountered one of the two
			previous cases, or
		\item $r''$ has moved near $r$, then $\dist{r}{r''}\leq D$, or
		\item the criteria used by $r$ to break up ties among several of its nearest neighbors
			makes it select $r''$ instead of $r'$, then $\dist{r}{r''} = D$.
		\end{enumerate}
		So, $D(\conf[t'])\leq D$ and we can rename $r'$ to $r''$ when iterating over the argument.
	\end{enumerate}
	Since the scheduler is fair, there is a time after which $r$ moves and $D(\conf)$ decreases
	by at least $\delta$. The rest follows.
	\qed
\end{proof}

\begin{theorem}
  \label{th:ft-det-gathering}
  In an $(n,f)$-crash system, where $n\!>\!f$, Algorithm~\ref{alg:ft-gathering}
  deterministically solves weak gathering
  under a fair centralized scheduler if robots are aware of
  multiplicity.
\end{theorem}
\begin{proof}
	Let us first prove that the Algorithm~\ref{alg:ft-gathering} satisfies the closure property of
	weak gathering, i.e.,
	to a gathered configuration follow only gathered configurations. Let $g_i$	be a gathered
	configuration.
	By definition of a gathered configuration, there is a unique point of maximal multiplicity that
	all correct robots occupy. Since, by construction of
	Algorithm~\ref{alg:ft-gathering}, the correct robots do not move,
	any subsequent configuration is gathered. This proves closure.

	\medskip
	By Proposition~\ref{prop:ft-det:maxmul}, the maximal multiplicity is nondecreasing.
	Let us now prove convergence by induction on the maximal multiplicity of configurations.
	
	Proposition~\ref{prop:ft-det:distinct2reachable} forms the basis of the induction by showing that
	a distinct configuration leads to a configuration with maximal multiplicity larger than one.
	
	\paragraph{Theorem~\ref{th:ft-det-gathering}; Induction step.}
	We show that, starting from a non-gathered configuration~\conf[x] with maximal multiplicity
	$M=\mulmax{\conf[x]}>1$, a configuration~\conf[y] with maximal multiplicity $M+1$ is eventually
	reached.
	
  %
  We say that a robot $p_b$ is \emph{blocked} if there are at least $\mulmax{c}\!-\!1$
	robots on the segment between $p_b$ and $q_b$, where $q_b$ is the castle%
	\footnote{C.f., definition of \emph{castle} and \emph{tower} in Section~\ref{sec:model:notation}.}
	that $p_b$ selects if it is active in configuration~\conf.
	Let \nCastles{\conf}, \nBlocked{\conf}, and \sumDist{\conf} respectively denote the number of castles,
	the number of blocked robots,	and the sum of distances from each robot to its nearest castle
	in configuration \conf.

	We can characterize a configuration~\conf by the quantities \mulmax{\conf}, \nCastles{\conf}, \nBlocked{\conf},
	and \sumDist{\conf}.
	Consider a configuration \conf[t] with $x\leq t <y$, characterized by $\mulmax{\conf[t]}=M$,
	\nCastles{\conf[t]}, \nBlocked{\conf[t]},	and \sumDist{\conf[t]},
  and consider all possible transitions from \conf[t] to the next configuration \conf[t+1].
	Since the scheduler is centralized, one correct robot, say $r$, is active in
	configuration \conf[t].
	We can summarize the transitions as follows:
	\begin{enumerate}
	\item $r$ is in a castle.
		\begin{enumerate}
		\item $\nCastles{\conf[t]}=1$.\\
			$r$ does not move. No change.
		\item $\nCastles{\conf[t]}>1$.\\
			$r$ aims for castle $q$.
			\begin{enumerate}
			\item $r$ is blocked. $r$ takes a side move.\\
				$\nCastles{\conf[t+1]} = \nCastles{\conf[t]}-1$.
			\item $r$ does not reach $q$.\\
				$\nCastles{\conf[t+1]} = \nCastles{\conf[t]}-1$.
			\item $r$ reaches $q$.\\
				GOAL: $\mulmax{\conf[t+1]}=M+1$ and $\nCastles{\conf[t+1]} =1$.
			\end{enumerate}
		\end{enumerate}
	\item $r$ is not in a castle (aims for castle $q$).
		\begin{enumerate}
		\item $r$ is blocked. $r$ takes a side move.\\
			$\nBlocked{\conf[t+1]} = \nBlocked{\conf[t]}-1$\\
			and $\sumDist{\conf[t+1]} < \sumDist{\conf[t]}$.
		\item $r$ does not reach $q$.\\
			$\sumDist{\conf[t+1]} = \sumDist{\conf[t]} - \delta$.
		\item $r$ reaches $q$.\\
			GOAL: $\mulmax{\conf[t+1]}=M+1$ and $\nCastles{\conf[t+1]} =1$.
		\end{enumerate}
	\end{enumerate}
	
	A first observation is that no new castle is created, in other words,
	\nCastles{\conf} never increases.
	
	As long as there is a castle with at least one correct robot, that robot
	is eventually active since the scheduler is fair. Provided that there
	are at least two castles, the number of castles decreases. This happens
	until the system reaches a configuration \conf[t'] ($x\leq t'<y$) in which any
	one of the following two conditions hold:
	\begin{itemize}
	\item There is one single castle.
	\item There are several castles, all of which consist only of crashed robots.
	\end{itemize}
	In either case, no robot already located in a castle moves.
	
	If the configuration is already gathered, convergence is proved, so assume
	it is not. Then, there must be some correct robot located outside of a
	castle, and that robot must eventually become active as the scheduler is
	fair.
	
	Let $r'$ be a correct robot located outside of a castle. Each time $r'$
	is active, it selects one of the castles $q'$ as its destination.
	In case another robot reaches a castle, the induction step is proved. So
	again assume that this is not the case.
	It follow that the number of castles
	and their locations do not change, thus $q'$ is the same castle	across
	activations of $r'$.
	
	Let $\Delta'$ be the distance between $r'$ in configuration \conf[t'].
	If $r'$ is initially blocked, then it performs a side move and $r'$ is
	no longer blocked (\nBlocked{\conf} decreases).
	Recall that, by construction, performing a side move does not increase the
	distance between the robot and its destination.

	After $\left\lceil\frac{\Delta'}{\delta}\right\rceil$ activations of $r'$,
	it reaches castle $q'$ thus increasing its multiplicity and proving the step.

	\medskip
	The induction ensures that, as long as a gathered configuration is not reached,
	the maximal multiplicity increases. The multiplicity cannot possibly be
	larger than the number of robots, so it follows that a gathered configuration is
	eventually attained.
	
	This proves convergence and, since we have proved
	closure before, the fact that Algorithm~\ref{alg:ft-gathering} solves
	weak gathering.
	\qed
\end{proof}





\subsubsection{Probabilistic weak gathering with multiple crashes}

In the remainder of this section, we show that weak gathering can be solved
probabilistically in an $(n,f)$-crash system (with $f < n$) under a fair scheduler.

Algorithm~\ref{alg:ft-prob-gathering} is a probabilistic algorithm constructed
on the deterministic Algorithm~\ref{alg:ft-gathering}. While the latter is for a
centralized scheduler, the former is for a fair scheduler, which allows robots
to be active simultaneously.

\begin{algorithm}
  \caption{Probabilistic fault-tolerant gathering for robot $p$ with
    multiplicity knowledge}
  \label{alg:ft-prob-gathering}
  \begin{small}
    \textbf{Functions}:\\
    \IND $\mulmax{\multiset{\Pi}}::$  the maximal multiplicity in \multiset{\Pi}.\\
    \IND $\MaxMult{\multiset{\Pi}}::$ the set of points with multiplicity \mulmax{\multiset{\Pi}}.\\

    \noindent
    \textbf{Actions}:\\
    \noindent
    \IND Observe(\multiset{\Pi}) :: $\mathit{true}$ 
    	$\longrightarrow$ \\
		\IND\IND\IND \textbf{if} $\point{p} \not\in \MaxMult{\multiset{\Pi}}$ \textbf{then} \hfill \emph{/* \point{p} not in a castle */}\\
		\IND\IND\IND\IND execute extended Algorithm~\ref{alg:ft-gathering};\\
 		\IND\IND\IND \textbf{else if} $\left|\MaxMult{\multiset{\Pi}}\right| = 1$ \textbf{then} \hfill \emph{/* unique castle */}\\
    \IND\IND\IND\IND stay;\\
    \IND\IND\IND \textbf{else}\hfill \emph{/* several castles */}\\
    \IND\IND\IND\IND \textbf{with} probability $\alpha=\min\left(\frac{1}{\mulmax{\multiset{\Pi}}},\frac{1}{2}\right)$ \textbf{do}\\
		\IND\IND\IND\IND\IND execute extended Algorithm~\ref{alg:ft-gathering};\\
    \IND\IND\IND\IND \textbf{otherwise}\\
    \IND\IND\IND\IND\IND stay;\\
	\end{small}
\end{algorithm}

The idea of the algorithm is that, in some situations (several castles or
distinct configurations), the simultaneous activation of several robots could
lead to endless oscillations of the system. For instance, given two robots
which are reachable and nearest from each other, activating them together would
lead to them swapping their positions. To prevent this situation from occurring
endlessly, the robots are required to first toss a coin and actually move only
upon success.

In addition, the side move performed in Algorithm~\ref{alg:ft-gathering} defines
a region from which a target point is selected arbitrarily.
Due to concurrent moves under Algorithm~\ref{alg:ft-prob-gathering}, an arbitrary
choice is no longer adequate.
Therefore, it becomes necessary to extend Algorithm~\ref{alg:ft-gathering} such that
the side move prevents two simultaneously moving robots from reaching the same location.
The choice of an appropriate target is guided by the following requirements:
\begin{itemize}
\item
	Let two robots $r$ and $r'$, initially collinear with castle~\castle{q}, select
  target points \point{T} and \point{T'}. Then, segments $\overline{r\point{T}}$ and
  $\overline{r'\point{T'}}$ intersect if and only if $r$ and $r'$ are collocated.
\end{itemize}
A construction that satisfies this requirement is presented in the appendix
(Sect.~\ref{sec:side-move:disamb}).
The probabilistic algorithm relies on Algorithm~\ref{alg:ft-det-gathering} extended
with a side move meeting those requirements.

\medskip
We first show that the convex hull of positions is non-increasing. This simple
result is important as one factor to ensure that the system does not oscillate.
\begin{proposition}
	\label{prop:convexhull}
	In an $(n,f)$-crash system, where $n\!>\!f$, with any scheduler and multiplicity detection,
	let \exec be an execution under Algorithm~\ref{alg:ft-prob-gathering}.
	
	Let \conf[t] and \conf[t'] be two configurations of \exec,
	and \Conv{\conf[t]} (respectively \Conv{\conf[t']}) the convex hull of robot locations in \conf[t] (resp. \conf[t']).
	Then, $t'>t \implies \Conv{\conf[t']} \subseteq \Conv{\conf[t]}$.
\end{proposition}
\begin{proof}
	Let $r$ be an arbitrary robot that moves through Algorithm~\ref{alg:ft-prob-gathering}:
	it can stay, move toward another robot, or perform a side move.
	
	In all three cases, the entire segment between $r$'s location and its target
	destination must be contained within the convex hull. 
	When $r$ stays, this holds trivially. When $r$ moves toward another robot,
	this holds because of the convexity of the convex hull. When $r$ performs a
	side move, this holds from the definition of the side move.
	
	As a result, no move can possibly bring a robot outside of the convex hull,
	which is thus non-increasing.
	\qed
\end{proof}
Notice however that the convex hull is not necessarily decreasing since the robots
located at the vertices of the convex hull could be crashed robots.

\medskip
We continue by proving important properties of executions under
Algorithm~\ref{alg:ft-prob-gathering}. The first proposition shows that, if the
number of castles can increase from one configuration to the next, then the maximal
multiplicity must necessarily have decreased.

\begin{proposition}
	\label{prop:ft-prob:castle++}
	In an $(n,f)$-crash system, where $n\!>\!f$, a fair scheduler and multiplicity detection,
	let $e$ be an execution under Algorithm~\ref{alg:ft-prob-gathering}.
	Let \conf[t] and \conf[t+1] be any two consecutive configurations in $e$.
	The number of castles increases in \conf[t+1] only if the maximal multiplicity decreases
	in \conf[t+1].
\end{proposition}
\begin{proof}
	Let $K$ denote the number of castles in \conf[t], and $M$ the maximal multiplicity in \conf[t].
	Suppose that there are $K+1$ castles in \conf[t+1].
		Now, assume by contradiction that the maximal multiplicity in \conf[t+1] is $M$ or more.
			Since there are $K+1$ castles in \conf[t+1], at least one castle of multiplicity $M$ or
			more must	have been created from $M$ independent robots (i.e., robots that did not
			belong to a castle in \conf[t]).
			
			Consider one of the robots, call it $r$, independent in \conf[t] and forming the new castle
			in \conf[t+1]. There are three possible cases.
			\begin{itemize}
			\item
				If $r$ did not move, then there must be another robot $r'$ that has moved, or else
				$r$ would not have been independent in \conf[t]. Then, consider the case of $r'$ instead.

			\item $r$ performed a straight move.
				By construction of the algorithm, there are less
				than $M$ independent robots on the segment between $r$ and its nearest castle, $r$
				included. But, by construction, no robot performing a side move can reach a ray
				containing robots performing straight moves.
			  
			\item $r$ performed a side move.
				The area targeted by the side move is convex, does not contain any robot, and does
				not contain any point reachable with a straight move.
				
				Hence, all robots collocated with $r$ in \conf[t+1] must have performed a side move.
				
				With the extended construction of the side move, every robot $r'$ collocated with
				$r$ in \conf[t+1] must have been collocated with $r$ in \conf[t]. Thus, $r$ was forming
				a castle in \conf[t]. A contradiction.
			\end{itemize}
		The maximal multiplicity in \conf[t+1] is not $M$ or more.
	The number of castles increases in \conf[t+1] only if the maximal multiplicity decreases
	in \conf[t+1].
	\qed
\end{proof}

\begin{proposition}
	\label{prop:ft-prob:one-castle}
	In an $(n,f)$-crash system, where $n\!>\!f$, a fair scheduler and multiplicity detection,
	let $e$ be an execution under Algorithm~\ref{alg:ft-prob-gathering}.
	If a configuration has a unique castle, then all configurations after
	that have only one castle and the maximal multiplicity is nondecreasing.
\end{proposition}
\begin{proof}
	Let \conf[t] be a configuration in $e$ with a unique castle. 
	
	  By construction of Algorithm~\ref{alg:ft-prob-gathering}, no robots in the castle move
	  when activated.
	  Thus, the maximal multiplicity does not decrease in configuration \conf[t+1].
		From Proposition~\ref{prop:ft-prob:castle++}, 
		the number of castles increases in \conf[t+1] only if the maximal multiplicity decreases
		in \conf[t+1].
		Therefore, the number of castles does not increase in \conf[t+1].
		
	The rest follows by induction on configurations.
	\qed
\end{proof}

An important consequence of these two propositions is that, when a configuration with a unique
castle is reached, then only configurations with a unique castle can follow. In other words,
distinct configurations or configurations with several castles can no longer occur.
%
We now additionally show that the system progresses deterministically to a gathered
configuration. As a result, we can later consider the formation of a unique castle to be
final, as it deterministically leads to gathering in finite steps.
\begin{proposition}
	\label{prop:ft-prob:single2gather}
	In an $(n,f)$-crash system, where $n\!>\!f$, a fair scheduler and multiplicity detection,
	any execution~$e$ (or execution suffix) under Algorithm~\ref{alg:ft-prob-gathering} that
	starts in a configuration with a unique castle leads to a gathered configuration in
	finite steps.
\end{proposition}
\begin{proof}
	Let \conf[t] be a configuration with a single castle~$\castle{Q}_t$ and maximal multiplicity
	$M=\mulmax{\conf[t]}\geq 2$. We prove that, either \conf[t] is gathered or there exists a
	configuration \conf[t'] in $e$ with $t<t'$ such that $\mulmax{\conf[t']}>M$.
	
		Since \conf[t] has a single castle, a robot located in the castle does not move when
		activated. Thus, only independent robots can move when active. In
		Algorithm~\ref{alg:ft-prob-gathering}, independent robots execute the first clause
		of the test, so the execution is deterministic and depends only on the activations
		of the scheduler.
		
		If there are no independent correct robots, then the configuration is	already
		gathered. Let us now consider the case when some correct robot is independent.
		Let $r$ be one such robot and let $D_t$ be the distance from $r$ to $\castle{Q}_t$ in
		configuration~\conf[t].
		
		The scheduler being fair, it must activate $r$ eventually. We consider two cases, depending
		whether $r$ is blocked or not in configuration \conf[t].
		\begin{itemize}
		\item If $r$ is not blocked, then it takes $\lceil D_t/\delta_r \rceil$ activations
			of $r$ to reach $\castle{Q}_t$, thus increasing the multiplicity of $\castle{Q}_t$, and hence maximal
			multiplicity, by at least one.
		\item
			If $r$ is blocked in \conf[t], then it performs a side move when it is activated.
			It is possible that all other robots blocked on the same ray as $r$ are activated
			at the same time, performing a side move in the same direction. Let $B_t$
			be the number of robots blocked on the same ray as $r$. $B_t$ is at
			most $n-2M$ because $M$ robots form castle $\castle{Q}_t$ and $M$ robots block the others
			on the ray (those may have crashed, so they will not necessarily move).

			Let all blocked robots move together with $r$, with $r$ being the farthest robot
			on the ray.
			After completing a side move, the number of blocked robots decreases by $M$. So,
			after at most $\frac{n}{M}-2$ side moves, $r$ is no longer blocked, and the rest
			follows from the first case.
		\end{itemize}
	
	\noindent
	This proves the claim and the remainder of the proof follows by induction on the
	maximal multiplicity.
	\qed
\end{proof}


%
%
%
We now show that the shortest distance between a robot and its nearest neighbor is also
non-increasing for Algorithm~\ref{alg:ft-prob-gathering}.

\begin{proposition}
	\label{prop:ft-prob:D:nonincreasing}
	Consider an execution of Algorithm~\ref{alg:ft-prob-gathering} under a fair scheduler.
	Let $D(\conf)$ be a function defined as the shortest distance between a robot and its
	nearest neighbors in configuration~\conf.
	Then, $D(\conf)$ is non-increasing.
\end{proposition}
\begin{proof}
	In distinct configurations, all robots execute only the third clause of the test of
	Algorithm~\ref{alg:ft-prob-gathering}.
	So, a robot either
	(1) executes Algorithm~\ref{alg:ft-det-gathering} and $D(\conf)$ is
	non-increasing by Proposition~\ref{prop:ft-det:D:nonincreasing},
	or (2) stays and $D(\conf)$ is non-increasing trivially.
	\qed
\end{proof}

%
%

%
%

\begin{proposition}
	\label{prop:ft-prob:distinct2castle}
	In an $(n,f)$-crash system, where $n\!>\!f$, with a fair scheduler and multiplicity
	detection, let \exec be an execution under Algorithm~\ref{alg:ft-prob-gathering}, and
	let \fragment{t}{} be any execution suffix of \exec starting in a distinct
	configuration~\conf[t].

	Then, with high probability, \fragment{t}{} contains a configuration 
	with maximal multiplicity larger than one.
\end{proposition}
\begin{proof}
	Given a distinct configuration \conf, let us first define its attractor
	graph \AG[{\conf}] to be a weighted directed graph in which each robot is a vertex,
	and such that, there is an arc from robot $r$ to robot $r'$ if and only if $r$ is
	not crashed in \conf and, upon activation in \conf, $r$ will select $r'$ as its
	target destination according
	to Algorithm~\ref{alg:ft-prob-gathering} (and by extension Algorithm~\ref{alg:ft-gathering}).
	The weight of an arc is given by the distance separating the two robots.
	We say that $r'$ is the attractor of $r$ in configuration \conf.
	
	Each path in \AG[{\conf}] has non-increasing weights and ends either in a cycle
	of equal weights or with a crashed robot. Since $n\!>\!f$, there is at least one
	robot that never crashes, and hence at least one path exists in every configuration
	of~\exec.
	
	Consider the execution fragment \fragment{t}{} starting in distinct configuration \conf[t],
	and take the extremity of one path in \AG[{\conf[t]}] such that the weight of the last
	arc(s) is minimal.
	Let us denote this weight by $\Delta(\conf[t])$, and consider independently the two
	possible situations regarding the extremity of the path.
	
	\begin{enumerate}
	\item The path ends with a crashed robot.\\
		Let $r'$ be the crashed robot and $r$ the last correct robot on the path, with $r'$ as attractor.
		Then, $\Delta(\conf[t])$ is the distance separating $r$ and $r'$ in \conf[t].
		
		Over successive activations in fragment \fragment{t}{}, three situations may occur.
		When we say that a robot gets close to $r$, we mean that there is a third robot $r''$ such
		that $\dist{r}{r''} \leq \dist{r}{r'}$ and hence $r$ changes its attractor from $r'$ to $r''$.

		\begin{enumerate}
		\item Robot $r$ does not crash and no other robot gets close to $r$.\\			
			After $\left\lceil \frac{\Delta(\conf[t])}{\delta_r} \right\rceil$ successful moves of $r$,
			$r$ reaches the location of $r'$, resulting in a configuration with	maximal multiplicity
			larger than one.
			
			The number of activations follows a binomial distribution, and hence this occurs after
			constant expected number of activations of $r$. The scheduler being fair, \fragment{t}{}
			contains a configuration with	maximal multiplicity larger than one with high probability.
			
		\item Robot $r$ crashes in configuration \conf[t'] with $t'>t$.\\
			We apply the same argument starting with configuration \conf[t'] and other robots.
			This happens at most $f-1$ times.

		\item Robot $r$ changes its attractor to another robot $r''$ in a configuration \conf[t'].\\
			In \conf[t'], the distance $\dist{r}{r''} \leq \dist{r}{r'}$. Take the new path in which
			$r$ is now involved and continue applying the argument over its extremity with
			$\Delta(\conf[t'])$ such that:
			\[
			\Delta(\conf[t']) \leq \dist{r}{r''} \leq \dist{r}{r'} \leq \Delta(\conf[t])
			\]
		\end{enumerate}
	
	\item The path ends in a cycle.\\
		The cycle involves $q$ non-crashed robots ($2\!\leq\!q\!\leq\!n$), all at distance $\Delta(\conf[t])$ to
		their attractor.
		
		Over successive activations in fragment \fragment{t}{}, there are several situations that may occur.

		\begin{enumerate}
		\item No robots crash and no external robot gets close.\\
			Each time some of the robots involved in the cycle are activated, the following
			situations may occur.
			\begin{enumerate}
			\item Some other robot in the cycle is not activated.\\
				Let $r$ be an activated robot with attractor $r'$, such that $r'$ is not activated.
				
				With probability at least $\frac{1}{n}$ robot $r$ moves (while $r'$ does not),
				and the cycle is broken. We apply the argument again starting with the new configuration
				\conf[t'], with $\Delta(\conf[t'])$ such that:
				\[
				\Delta(\conf[t']) \leq \dist{r}{r'} - \delta_r \leq \Delta(\conf[t]) - \delta_r
				\]
			
			\item All robots in the cycle are activated.\\
				There are three sub-cases:
				\begin{itemize}
				\item With probability $\left(1-\frac{1}{n}\right)^q$ no robots move.
					The situation does not change.
					
				\item With probability $\left(\frac{1}{n}\right)^q$ all robots move.
					The situation remains if and only if (1)~all robots $r_i$ involved in the cycle
					have the same reachable distance $\delta_{r_i}$,
					and (2)~$\Delta[{\conf[t]}] = \delta_{r_i}$.
					
					In all other cases, $\Delta[{\conf[t+1]}] < \Delta[{\conf[t]}]$.

				\item With remaining probability, a strict subset of the robots move and the other don't.
					This is identical to the previous case, when some robot is not activated.
					This results in the cycle being broken, and we apply the argument again
					starting with the new configuration \conf[t'], with $\Delta(\conf[t'])$ such that:
					\[
					\Delta(\conf[t']) \leq \dist{r}{r'} - \delta_r \leq \Delta(\conf[t]) - \delta_r
					\]
				\end{itemize}
					
			\end{enumerate}

		\item Some robot $r$ in the cycle crashes in configuration \conf[t'] with $t'>t$.\\
			The path no longer ends in a cycle, and we apply the argument starting in configuration \conf[t'] and with $\Delta(\conf[t'])$ such that
			$\Delta(\conf[t']) \leq \Delta(\conf[t])$.
			
		\item Some robot $r$ changes its attractor to another robot $r''$ in a configuration \conf[t'].\\
			This breaks the cycle and defines a new path involving $r$. We apply the argument over the
			extremity of this path, with $\Delta(\conf[t'])$ such that
			\[
			\Delta(\conf[t']) \leq \dist{r}{r''} \leq \Delta(\conf[t])
			\]
		\end{enumerate}

	\end{enumerate}
	Regardless of scheduler choices, the minimal distance from a non-crashed robot
	to its attractor eventually decreases and, with high probability, the system
	reaches a configuration	with multiplicity larger than one.
	\qed
\end{proof}


\begin{lemma}
	\label{prop:ft-prob:multiple2single}
	In an $(n,f)$-crash system, where $n\!>\!f$, with a fair scheduler and
	multiplicity detection, let $e$ be an execution under
	Algorithm~\ref{alg:ft-prob-gathering}.
	Let $e'$ be any execution suffix starting in a configuration with multiple castles.
	Then, with high probability, $e'$ contains a configuration with a single castle.
\end{lemma}
\begin{proof}
	Let \conf be a configuration with $K>1$ castles of multiplicity $M$ in $e'$, and let us
	calculate the probability to reach a configuration $\conf'$ with $K'$ castles of
	multiplicity $M'$ after the next activation.
	
	Let \set{K} denote the set of castles in \conf. For each castle $\castle{k}\in\set{K}$,
	let $\mathit{i}_k$ denote the number of (incoming) robots that can enter castle \castle{k}
	upon activation. To be counted, a robot must be correct, located outside castle \castle{k},
	activated by the scheduler in configuration \conf, have castle \castle{k} as its destination,
	and be able to reach \castle{k} in one step.
	Similarly, let $\mathit{o}_k$ denote the number of (outgoing) robots that can leave
	castle \castle{k} upon activation. To be counted, a robot must be correct, located inside
	castle \castle{k}, and activated by the scheduler.
	Note that, when castles are near, a single robot may be counted simultaneously as
	an outgoing robot of some castle and an incoming robot of another castle.
	
	We now define a function $\textsc{Balance}(i,o)$ to calculate the probability that
	the movement of $i$ incoming robots and $o$ outgoing robots exactly compensate
	each other. This is given by the probability that the same number of incoming and
	outgoing robots move, so that every departure of an outgoing robot is compensated
	by the arrival of an incoming one.
	\begin{align*}
		&\textsc{Balance}(i,o) = \PP[\mbox{none move}] + \sum_{m=1}^{\min(i,o)}\PP[\mbox{$m$ arrive/depart}]
		\\
		&=
			(1\!-\!\frac{1}{M})^{i+o}
		 	+ \sum_{m=1}^{\min(i,o)}
			\left[
				\begin{array}{ll}
					\displaystyle
		  		{i \choose m} (\frac{1}{M})^m (1\!-\!\frac{1}{M})^{i-m}
					\\
					\displaystyle
		  		\times {o \choose m} (\frac{1}{M})^m (1\!-\!\frac{1}{M})^{o-m}
				\end{array}
			\right]
		\\
		&=
			(1\!-\!\frac{1}{M})^{i+o}
		 	+ \sum_{m=1}^{\min(i,o)}
			\left[
		  	{i \choose m}{o \choose m}
				(      \frac{1}{M})^{2m}
				(1\!-\!\frac{1}{M})^{(i\!+\!o\!-\!2m)}
			\right]
		\\
 		&=
			(1\!-\!\frac{1}{M})^{i+o}
		\left(1 + 
		 	\sum_{m=1}^{\min(i,o)}
		  {i \choose m} {o \choose m}
		  (  \frac{1}{M-1})^{2m}
		\right)
		\\
	\end{align*}
	
	We define the function $\textsc{Increase}(i,o,x)$ to return the probability that
	the multiplicity of a castle increases by $x$ in the presence of $i$ incoming robots
	and $o$ outgoing robots. This is the probability that $x$ incoming robots move with
	the remaining incoming and outgoing robots compensating each other's movements.
	\begin{equation*}
		\textsc{Increase}(i,o,x) =
		  {i \choose x} (\frac{1}{M})^x \cdot \textsc{Balance}(i-x, o)
	\end{equation*}
	
	Let $\PP_{inc}(K',x)$ return the probability that configuration $\conf'$ has exactly
	$K'$ castles of multiplicity $M'=M+x$. That probability can be expressed as the
	probability that any subset \set{K'} of $K'$ castles increase their multiplicity
	by $x$ and all remaining castles do not increase multiplicity to any value $x'$
	larger or equal to $x$. Let $\set{P}_{=K'}(\set{K})$ denote the set of subsets
	of \set{K} of cardinality~$K'$, and we can express $\PP_{inc}(K',x)$ as follows.
	\[
		\PP_{inc}(K',x)
		\!=\!
		\left[
		\begin{array}{l}
			\displaystyle
			\prod_{\set{K'}\in\set{P}_{=K'}(\set{K})} ~
				\prod_{k'\in\set{K'}}
					\textsc{Increase}(i_{k'}, o_{k'}, x)
		\\
			\displaystyle
			\times\!
			\prod_{k''\in\set{K}\setminus\set{K'}} ~
				\prod_{x'=x}^{i_{k''}}
					\left( 1\!-\!\textsc{Increase}(i_{k''},\,o_{k''},\,x') \right)
		\end{array}
		\right]
	\]
	
	The probability of having a single castle in the next
	configuration is obviously at least as high as having
	a single castle by increasing the multiplicity of one
	of them by one.	So, we can state the following inequality
	\[
		\PP\left[\mbox{$\conf'$ has one single castle}\right]
		\geq
		\PP_{inc}(K'=1,x=1)
	\]
	The exact probability must consider	increases of the multiplicity
	by more than one, and the change in number of castles due to a
	decrease of the multiplicity. However, this is sufficient for the
	proof since, as we are not concerned here with measuring an actual
	convergence rate, the mere existence of a transition with positive
	probability is sufficient.
	
	Let us consider the configurations for which $\PP_{inc}(K'=1,x=1)$
	is zero. From the formula obtained for $\PP_{inc}$, we see that it
	is zero when, for all castle $k$, $i_k$ is zero.
	This can occur in several situations.
	\begin{itemize}
	\item All robots have crashed.\\
		This contradicts the assumption that $f<n$ which implies that
		there is at least one correct robot (i.e., a robots that never
		fails).
	\item The ``near'' robots are never activated.\\
		This contradicts the assumption that the scheduler is fair. If
		some ``near'' and correct robots exist, they must be activated
		eventually.
	\item There are no ``near'' robots.\\
		When a ``far'' robot $r$ is activated, its distance to the nearest
		castle decreases by $\delta_r$ (to simplify the discussion
		we omit the case of the side move). Thus, either the
		configurations of castle change or $r$ becomes a ``near'' robot.
	\end{itemize}
	So, with high probability, $\exec'$ contains a configuration with a single castle.
	\qed
\end{proof}

\begin{theorem}
  \label{th:wg}
  In an $(n,f)$-crash system, where $n\!>\!f$, Algorithm~\ref{alg:ft-prob-gathering}
  probabilistically solves weak gathering
  under a fair scheduler if robots are aware of
  multiplicity.
\end{theorem}

\begin{proof}
	Closure is satisfied by Algorithm~\ref{alg:ft-prob-gathering} because,
	in a gathered configuration, all correct robots are by definition located
	on a unique castle, and hence do not move when activated. Thus, a gathered
	configuration always follows after a gathered configuration and closure is
	satisfied.
	
	\medskip
	To show convergence, let us consider an adversary~\Adv, as defined by
	(1)~an initial configuration,
	(2)~an activation strategy,
	and (3)~control of robot crashes.
	However, \Adv	has no control on random choices made by robots, and no prior
	knowledge of their outcomes.
	The goal of \Adv is then to construct an infinite execution $\bar\epsilon$
	that contains no gathered configurations, and such that $\bar\epsilon$
	occurs with non-zero probability.
	
	From Proposition~\ref{prop:ft-prob:single2gather}, the formation of a single
	castle leads to a gathered configuration. So, \Adv must prevent the formation
	of a single castle.
	Let us now focus on the number of castles in each configuration and look at
	the transitions when this changes. Figure~\ref{fig:ft-prob:markov} depicts a
	Markov chain that represents the changes in the number of castles. The chain
	provides a conservative estimation by integrating simplifications that
	systematically favor \Adv.
	Since we are not concerned here with measuring the actual convergence rate, the mere
	existence of transitions with a positive probability is sufficient.
	We now describe its construction.
	
	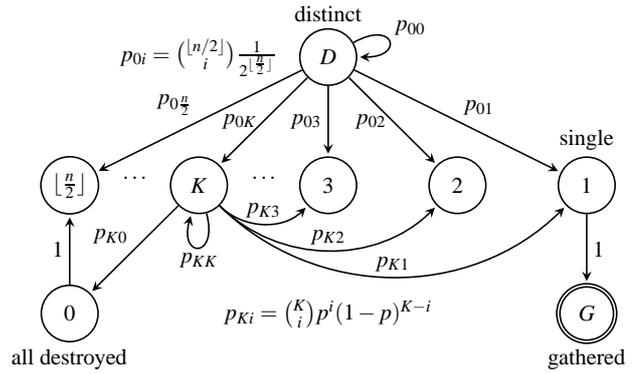
\begin{figure}
		\centering
		\begin{tikzpicture}[->,>=stealth,shorten >=1pt,%
	                    auto,node distance=1.7cm,on grid,semithick,
	                    inner sep=2pt,bend angle=10]
			\node[state,label=above:{single}] (1) {$1$};
			\node[state] (2)   [left of=1] {$2$};
			\node[state] (3)   [left of=2] {$3$};
			\node[state] (K)   [left of=3] {$K$};
			\node[state] (n/2) [left of=K] {$\lfloor\frac{n}{2}\rfloor$};
			
			\node[state,label=above:{distinct}]           (D) [above of=3]   {$D$};
			\node[state,label=below:{all destroyed}]      (0) [below of=n/2]  {$0$};
			\node[state,accepting,label=below:{gathered}] (G) [below of=1] {$G$};
			
			\node (formula)  [below of=3] {$p_{Ki}={K \choose i}p^i(1-p)^{K-i}$};
			\node (formula2) [left of=D]  {$p_{0i}={{\lfloor n/2\rfloor} \choose i} \frac{1}{2^{\lfloor\frac{n}{2}\rfloor}}$};
			
			\path[->] (0) edge node[left]  {$1$} (n/2);
			\path[->] (1) edge node[right] {$1$} (G);
			
			\path (K)   -- node{$\ldots$} (3);
			\path (n/2) -- node{$\ldots$} (K);
			
			\path[->,bend angle=45]
								(K) edge[bend right] node[above] {$p_{K3}$} (3)
										edge[bend right] node[above] {$p_{K2}$} (2)
										edge[bend right] node[above] {$p_{K1}$} (1)
										edge[loop below] node[below] {$p_{KK}$} (K)
										edge             node[above left] {$p_{K0}$} (0);
			
			\path[->] (D) edge[in=0,out=30,loop] node {$p_{00}$} (D)
										edge node {$p_{01}$} (1)
										edge node[left] {$p_{02}$} (2)
										edge node[left] {$p_{03}$} (3)
										edge node[left] {$p_{0K}$} (K)
										edge node[above left] {$p_{0\frac{n}{2}}$} (n/2);
												
		\end{tikzpicture}
		\caption{%
			Markov chain representing the transitions of changes in the number of castles.
			A number represents the number of castles in the configurations. For every
			$K\in\{2,\ldots,\lfloor\frac{n}{2}\rfloor\}$, outgoing transitions follow a binomial
			distribution (only $K$ depicted). Transitions from $K$ to distinct are ignored
			because transitioning to state $0$ instead favors an adversary.
			State~$1$ (single castle) leads to a gathered configuration, which is absorbing.
			When all castles are destroyed, a worst-case choice leads to
			$\lfloor\frac{n}{2}\rfloor$ castles (e.g., with lower multiplicity) in the next
			configuration.
		}
		\label{fig:ft-prob:markov}
	\end{figure}

	\medskip
	Assume first that the system is in a distinct configuration. From
	Proposition~\ref{prop:ft-prob:distinct2castle}, \Adv cannot prevent
	the formation of castles. It can however control activations so that
	several castles are formed simultaneously.
	To maximize the chance of creating multiple castles, \Adv can	postpone the
	activations of every robot that can reach its nearest neighbor,	until all
	robots form pairs%
	\footnote{
		Situations in which robots form a chain or a cycle result in the creation
		of fewer castles, which is less favorable to the adversary~\Adv.
	}
	of mutually nearest neighbors. Then, all robots are activated and move with
	probability $\frac{1}{2}$, resulting in a number of castles that follows
	a binomial distribution $B(\lfloor\frac{n}{2}\rfloor,\frac{1}{2})$.
	\begin{equation*}
		\PP\left[\mbox{$x$ castles created}\right] = p_{0x} =
		{\lfloor\frac{n}{2}\rfloor \choose x} \frac{1}{2^{\lfloor\frac{n}{2}\rfloor}}
	\end{equation*}
	When no castles are created, the resulting configuration is distinct and the
	process repeats itself.
	
	\medskip
	Assume now that the system is in a configuration with $K>1$ castles. The number
	of castles can change in two possible ways: (1)~independent robots moving
	inside a castle, or (2)~robots leaving a castle thus destroying it.
	
	
	When independent robots move inside a castle, no additional castle can be
	created in the next configuration (from Proposition~\ref{prop:ft-prob:castle++}).
	Looking at the best case (for \Adv)	when one independent robot is ready to
	move inside every castle, we obtain that the probability of
	castle creation	follows a binomial distribution $B(K,\frac{1}{M})$.
	\[
		\PP\left[\mbox{$x$ castles created}\right] = p_{Kx} =
		{K \choose x} (\frac{1}{M})^x (1-\frac{1}{M})^{K-x}
	\]
	When no castles are created, the resulting configuration is identical and the
	situation is repeated.

	When robots leaving castles result in their destruction, there can be three
	possible outcomes in the configuration that follows:
	\begin{itemize}
	\item Several castles remain. There can be no more than $K$ castles.
	\item A single castle remain. This is the situation that \Adv must avoid.
	\item All castle destroyed. The next configuration has a lower maximal
		multiplicity, and can result in a larger number of castles of
		multiplicity lower than $M$.
	\end{itemize}
	The probability that a given castle is not destroyed by some robot moving
	outside has the following probability
	\[
		\PP\left[\mbox{castle not destroyed}\right] = \alpha^M = (\frac{1}{M})^M = p(M)
	\]
	For a configuration to have multiple castles, $M$ must necessarily be between $2$
	and	$\lfloor\frac{n-1}{2}\rfloor$. For any of these values, both $p(M)$ and $1-p(M)$
	are strictly positive. To simplify the model (Fig.~\ref{fig:ft-prob:markov}), we
	assume that it is a positive constant $p$ with $0<p<1$, chosen to be the value that
	favors the adversary most, and that does not depend on multiplicity. Again, its exact
	value is secondary, as long as it is strictly positive for any value of $K$, $M$,
	and $n$ finite.

	
	The number of castles in the following configuration follows a binomial
	distribution $B(K,\alpha^M)$
	\[
		\PP\left[\mbox{$x$ castles size $M$ remain}\right]
		= p_{Kx} = {K \choose x} p^x (1-p)^{K-x}
	\]
	When no castles remain, the maximal multiplicity decreases and a larger number of castles
	of lower multiplicity may be created. Unless the next configuration is distinct, there can
	be no more than $\lfloor\frac{n}{2}\rfloor$ castles in the resulting configuration.
	We observe that, from the viewpoint of the adversary \Adv, the best case is when the maximal
	number of castles are formed. So, we assume that the destruction of all castles in a
	configuration always leads to a configuration with the maximal number of castles.
	
	\medskip
	Putting this together gives us the Markov chain depicted in Fig.~\ref{fig:ft-prob:markov}.
	For configurations with multiple castles, the figure shows only the transitions from state~$K$.
	According to Proposition~\ref{prop:ft-prob:single2gather}, configurations with a single castle
	lead to a gathered configuration with probability~1.
	
	The resulting Markov chain contains a single absorbing state $G$. It is	a well-known result
	that, in an absorbing Markov chain, the process will be absorbed with	probability~1. Since
	the only absorbing state is $G$ (gathered), convergence is satisfied with probability~1.
	\qed
\end{proof}

%% file: btss.tex
\section{Byzantine Tolerant and Self-stabilizing Gathering}
\label{sec:byzt-gathering}
We study now the feasibility of gathering in systems prone
to Byzantine failures.

Agmon and Peleg~\cite{AP06} proved the impossibility of weak
gathering in a $(3,1)$-Byzantine system under a fair scheduler 
(Theorem~\ref{th:ap:byzimposs}). The result applies to both SYm
and CORDA models. The following lemma proves that the impossibility
still holds under a round-robin scheduler, and even if the algorithm
is probabilistic.

\begin{lemma}
  \label{lemma:impbyz}
  In a Byzantine-prone system, there is no deterministic or probabilistic algorithm
  that solves $(n,f)$-weak gathering, $f\geq 1$ and $n>f+1$, under a round-robin scheduler
  without additional assumptions.
\end{lemma}
\begin{proof}
	By contradiction, let \algo{A} be an algorithm that solves gathering.
	Assume that a single robot $r_B$ is Byzantine (or let the other
	Byzantine robots behave like correct ones).
	Let \algo{A} execute normally until all robots share the same location~\point{P}.
	When activated, let $r_B$ move to a second location $P'$ selected as follows:
	\begin{itemize}
	\item if \algo{A} is \emph{deterministic}, chose \point{P'} such that,
		applying the criteria used in \algo{A} when selecting a target location,
		some correct robot $r$ will move to \point{P'}.
	\item if \algo{A} is \emph{probabilistic}, chose any $\point{P'}\not=\point{P}$.
	\end{itemize}
	In either case, a correct robot $r$ must move because, robots being oblivious, they have
	no way to know that gathering was already achieved. Furthermore, in the absence
	of multiplicity detection, there is no way to distinguish \point{P} and \point{P'} by their
	multiplicity.
	Since there are at least two correct robots ($n>f+1$) and the scheduler is
	centralized, the move of $r$ toward \point{P'} results in a non-gathered configuration.
	
	The situation can be repeated each time the system is in a gathered configuration.
	This clearly violates the closure property of weak gathering, since closure
	requires that any execution suffix starting in a gathered configuration contains
	only gathered configurations. Thus, \algo{A} does not solve weak gathering.
	\qed
\end{proof}

\subsection{Deterministic Byzantine Gathering}
%
%
%
%
The following lemma shows that if the power of the scheduler is increased, weak gathering
is impossible in a $(3,1)$-Byzantine system, even if robots are aware of the
system multiplicity.

\begin{lemma}
  \label{lemma:(3,1)byz}
  In a $(3,1)$-Byzantine system, there is no deterministic algorithm
  that solves weak gathering under a fair centralized $k$-bounded
  scheduler with $k \geq 2$, even if robots are aware of multiplicity.
\end{lemma}
\begin{proof}
%
Assume an arbitrary initial configuration, a configuration where robots occupy distinct positions. 
The general proof idea is the following : the 
byzantine node plays the attractor role, hence the system never reaches a terminal configuration.
Consider a schedule $Sch$ such that after each execution of a correct robot the scheduler 
gives the permission to the byzantine robot to move. This schedule verifies the 
specification of the $2$-bounded scheduler. Assume that  
each time a correct node chooses to move, it chooses as 
target the location of the Byzantine node. Then, following the scheduler $Sch$ 
the Byzantine node will replace 
the location of the node that just joined its location. Therefore, the system never converges.
\end{proof}

%
%
%
%
%

The following lemma establishes a lower bound for the fair centralized bounded
scheduler that prevents the deterministic gathering.

\begin{lemma}
  \label{lemma:byz511:even}
  In an $(n,f)$-Byzantine system with $n$ even and $f \geq 1$, there is
  no deterministic algorithm that solves weak gathering under a fair centralized
  $k$-bounded scheduler with $k \geq \left\lceil\frac{n-f}{f}\right\rceil$,
  even if robots are aware of multiplicity.
\end{lemma}

\begin{proof}
  Consider an initial configuration such that the configuration is bivalent
  with two locations having equal multiplicity ($n$ is even), and robots are reachable from
  each other.
  Assume that the Byzantine robots are spread evenly between the two locations.
  Let $g_1$ and $g_2$ be the two groups, such that, if $f$ is odd, $g_1$
  has one more Byzantine robot than $g_2$.
	Consider the following activation schedule.
	\begin{itemize}
	\item Activate a correct robot in $g_2$: it must necessarily move to $g_1$ or else
	no execution could possibly reach gathering.
	\item	Activate a Byzantine robot in $g_1$, and let it move to $g_2$.
	The resulting configuration is symmetrical to the original one.
	\end{itemize}
	By repeating the same sequence, a Byzantine robot counterbalances every move of
	a	correct robot, and the system is always in a bivalent configuration. 
  
  Since there is a total of $n\!-\!f$~correct robots and $f$ Byzantine robots,
  the adversary can distribute the moves between the Byzantine robots. Thus,
  between each consecutive activation of a correct robot, the adversary
  must activate a Byzantine robot only
  $\left\lceil\frac{n-f}{f}\right\rceil$ times.
	\qed
\end{proof}

The following lemma states a lower bound for a bounded scheduler
that prevents deterministic gathering.
\begin{lemma}
  \label{lemma:byz511:odd}
  In an $(n,f)$-Byzantine system with $n$ odd and $f \geq 2$, there is
  no deterministic
  algorithm that solves weak gathering under a fair centralized
  $k$-bounded scheduler with $k \geq \left\lceil\frac{n-f}{f-1}\right\rceil$,
  even if robots are aware of multiplicity.
\end{lemma}

\begin{proof}
	Let the initial configuration be a bivalent configuration such that
	robots are reachable from each other and the multiplicity of the
	two locations differ by one.
	
	Let $g_a$ be the small group and $g_b$ the big one. Let all $f$ Byzantine
	robots be in $g_b$. If there are more than half Byzantine robots, then
	simply let all robots in $g_b$ be Byzantine ones.
	
	Now, call one of the Byzantine robots in $g_b$ the switch $r_{sw}$, and	
	consider the following schedule:
	\begin{enumerate}
	\item Activate a correct robot in $g_a$, say $r$.
	  It must be instructed to move to the other point of multiplicity, or
	  else gathering would not possibly be achieved in a fault-free case. 
	\item Each time a correct robot moves to $g_b$, activate a Byzantine
	  robot in $g_b$ (except the switch $r_{sw}$), and let it move to $g_a$.
	\item Repeat the procedure until one of the following condition holds:
	(1)~all correct robots originally in $g_a$ have moved, or 
	(2)~all Byzantine robots originally in $g_b$ have moved, except $r_{sw}$.
	
	\item Move $r_{sw}$ to $g_a$, which becomes now the larger group.
	\item Repeat the procedure with correct robots in $g_b$ so that they
	  move to $g_a$.
	\end{enumerate}
	At each iteration of the procedure, $f\!-\!1$ correct robots move from one
	group to the other, while $f\!-\!1$ Byzantine robots negate their move.
	
	Thus, a Byzantine robot needs to be activated at most
	$\left\lceil\frac{n-f}{f-1}\right\rceil$ times between two consecutive
	activations of a correct robot, .
	\qed
\end{proof}

%% file: summary.tex
\section{Summary}
\label{sec:summary}

We have summarized most of the theorems, their relationships, and their scope into tables (Table~\ref{tab:sum:strong} and~\ref{tab:sum:weak}).
Results are grouped according to the problem (strong or weak gathering) and the fault models:
strong gathering in fault-free (Table~\ref{tab:sum:none}) and single crash (Table~\ref{tab:sum:strong-crash}) environments;
as well as weak gathering in multiple crashes (Table~\ref{tab:sum:weak-crash}) and single Byzantine environments (Table~\ref{tab:sum:Byz}).

\newcommand{\YES}{\ensuremath{\bigcirc}\xspace}
\newcommand{\NO}{\ensuremath{\times}\xspace}

\newcommand{\scU}{unfair}
\newcommand{\scUC}{unfair centr.}
\newcommand{\scF}{fair}
\newcommand{\scFC}{fair centr.}
\newcommand{\scFB}{fair bounded}
\newcommand{\scFkB}{fair $k$-bounded}
\newcommand{\scFBR}{fair 1-bounded}
\newcommand{\scFCkB}{fair centr. $k$-bounded}
\newcommand{\scFTBC}{fair 2-bounded centr.}
\newcommand{\scFCBR}{round-robin}
\newcommand{\scProb}{probabilitistic}
\newcommand{\scFSyn}{fully synchronized}

\newcommand{\alDET}{det.}
\newcommand{\alPRO}{prob.}

\newcommand{\im}[1]{\cellcolor[gray]{0.8}NO(\textit{#1})}
\newcommand{\ok}[1]{OK(\textit{#1})}
\newcommand{\IM}[1]{\cellcolor[gray]{0.8}\textbf{NO({#1})}}
\newcommand{\OK}[1]{\textbf{OK({#1})}}
\newcommand{\spec}[1]{\cellcolor[gray]{0.9}NO/?(\textit{#1})}
\newcommand{\SPEC}[1]{\cellcolor[gray]{0.9}\textbf{NO/?({#1})}}

\newcommand{\DUB}[1]{{\color{red}\textbf{#1}}}

\newcommand{\LrA}{Th.\ref{th:pre:nodet2robots}}
\newcommand{\LrB}{Th.\ref{th:pre:pos3robots}}
\newcommand{\LrC}{Th.\ref{th:pre:nomultimposs}}
\newcommand{\LrD}{Th.\ref{th:ap:crashpos}}
\newcommand{\LrE}{Th.\ref{th:ap:byzimposs}}
\newcommand{\LrF}{Th.\ref{th:ap:byzfsyncposs}}
\newcommand{\LrG}{Th.\ref{th:ap:byz3syncposs}}
\newcommand{\LrH}{L.\ref{th:aphyp:byzimposs}}

\newcommand{\LaA}{Th.\ref{th:detimpossibility-ff}/\ref{th:detimpossibility-ff:distinct}}
\newcommand{\LaB}{L.\ref{lem:prob-2gathering}}
\newcommand{\LaC}{L.\ref{lem:sim:centr-2gathering}}
\newcommand{\LaD}{L.\ref{lem:sim:centr-n-imposs}}
\newcommand{\LaE}{L.\ref{lem:sim:bound-2gathering}}
\newcommand{\LaF}{L.\ref{lem:probimpossibility-ff}}
\newcommand{\LaG}{Th.\ref{th:probpossibility}}
\newcommand{\LaH}{Th.\ref{th:detimpossibility:k2}}

\newcommand{\LbA}{L.\ref{(3,1)gathering}}
\newcommand{\LbB}{C.\ref{lem:imp-ff-gathering}}
\newcommand{\LbC}{C.\ref{lemma:(n,1)ip}}
\newcommand{\LbD}{L.\ref{lemma:coralg-ff-gathering}}
\newcommand{\LbE}{L.\ref{(3,1)pstrong-gathering}}
\newcommand{\LbF}{L.\ref{lem:prob:(2,1)strong}}
\newcommand{\LbG}{L.\ref{lem:det:(2,1)strong}}

\newcommand{\LcA}{L.\ref{lemma:impweak}}
\newcommand{\LcB}{Th.\ref{th:ft-det-gathering}}
\newcommand{\LcC}{Th.\ref{th:wg}}
\newcommand{\LcD}{N. \ref{note:imp:crash:(n,1)}}

\newcommand{\LdA}{L.\ref{lemma:impbyz}}
\newcommand{\LdB}{N.\ref{note:31gathering:centreg}}
\newcommand{\LdC}{L.\ref{lemma:(3,1)byz}}
\newcommand{\LdD}{Wrong note}
\newcommand{\LdE}{L.\ref{lemma-bound}}
\newcommand{\LdF}{C.\ref{cor:n1gathering:centr:odd}}
\newcommand{\LdG}{\yyy}
\newcommand{\LdH}{L.\ref{lemma:byz511:even}} 
\newcommand{\LdI}{L.\ref{lemma:byz511:odd}}  
\newcommand{\LdJ}{L.\ref{lemma:corrbyz}}

\newcommand{\xxx}{{\color{red}\emph{xxx}}}
\newcommand{\yyy}{{\color{red}\emph{yyy}}}
\newcommand{\zzz}{{\color{red}\emph{zzz}}}

\begin{table*}
  \centering
  \caption{Strong gathering problem}
  \label{tab:sum:strong}
\subfloat[Fault-free model]{
  \label{tab:sum:none}
  \begin{tabular}{|c:c;{2pt/2pt}c:c||c||c:c;{2pt/2pt}c:c|}
    \hline
    \multicolumn{4}{|c||}{multiplicity}
    &
    & \multicolumn{4}{c|}{without multiplicity}
    \\\cline{1-4}\cline{6-9}
    \multicolumn{2}{|c;{2pt/2pt}}{deterministic} &  \multicolumn{2}{c||}{probabilistic} 
    &  \textbf{Scheduler}
    & \multicolumn{2}{c;{2pt/2pt}}{deterministic} &  \multicolumn{2}{c|}{probabilistic} 
    \\\cline{1-4}\cline{6-9}
 	$n=2$    &$n\geq 3$ & $n=2$    &$n\geq 3$
    &
    & $n=2$    &$n\geq 3$ & $n=2$    &$n\geq 3$
    \\\hline\hline
    \im{\LrA} &          &\OK{\LaB} &
    &\scU
    &\im{\LrA}  &\im{\LrC} &\OK{\LaB} &\im{\LaF}
    \\\hdashline
    \OK{\LaC} &          &\ok{\LaB} &
    &\scUC
    &\OK{\LaC} &\im{\LaH} &\ok{\LaB} &\im{\LaF}
    \\\hdashline
    \IM{\LrA} &\OK{\LrB} &\ok{\LaB} &\ok{\LrB}
    &\scF
    &\IM{\LrA} &\IM{\LrC} &\ok{\LaB} &\im{\LaF}
    \\\hdashline
    \ok{\LaC} &\ok{\LrB} &\ok{\LaB} &\ok{\LrB}
    &\scFC
    &\ok{\LaC} &\im{\LaH} &\ok{\LaB} &\IM{\LaF}
    \\\hdashline
    \im{\LrA} &\ok{\LrB} &\ok{\LaB} &\ok{\LrB}
    &\scFkB
    &\im{\LrA} &\im{\LaH} &\ok{\LaB} &\OK{\LaG}
    \\\hdashline
    \ok{\LaC} &\ok{\LrB} &\ok{\LaB} &\ok{\LrB}
    &\scFTBC
    &\ok{\LaC} &\IM{\LaH} &\ok{\LaB} &\ok{\LaG}
    \\\hdashline
    \IM{\LrA} &\ok{\LrB} &\ok{\LaB} &\ok{\LrB}
    &\scFBR
    &\IM{\LrA} &\spec{\LaA}$^a$ &\ok{\LaB} &\ok{\LaG}
    \\\hdashline
    \ok{\LaC} &\ok{\LrB} &\ok{\LaB} &\ok{\LrB}
    &\scFCBR
    &\ok{\LaC} &\SPEC{\LaA}$^a$ &\ok{\LaB} &\ok{\LaG}
    \\\hline
    \multicolumn{9}{c}{
			\begin{minipage}{.95\textwidth}
	    \footnotesize
	    \begin{trivlist}
	      \item[$^a$]{
	      	Special: Th.~\ref{th:detimpossibility-ff} proves the impossibility of self-stabilizing gathering.
					Th.~\ref{th:detimpossibility-ff:distinct} proves it for gathering \emph{provided Conjecture~\ref{cj:criteria} holds}.
				}
	    \end{trivlist}
	  	\end{minipage}
    }
  \end{tabular}
}

\subfloat[Crash model; $f=1$]{
  \label{tab:sum:strong-crash}
  \begin{tabular}{|c:c;{2pt/2pt}c:c||c||c:c;{2pt/2pt}c:c|}
    \hline
    \multicolumn{4}{|c||}{multiplicity}
    &
    & \multicolumn{4}{c|}{without multiplicity ($f=1$)}
    \\\cline{1-4}\cline{6-9}
    \multicolumn{2}{|c;{2pt/2pt}}{deterministic} & \multicolumn{2}{c||}{probabilistic}
    & \textbf{Scheduler}
    & \multicolumn{2}{c;{2pt/2pt}}{deterministic} & \multicolumn{2}{c|}{probabilistic} 
    \\\cline{1-4}\cline{6-9}
    $n=2$    &$n\geq 3$ & $n=2$    &$n\geq 3$ 
    &
    & $n=2$    &$n\geq 3$ & $n=2$    &$n\geq 3$ 
    \\\hline\hline
    \im{\LrA} &\im{\LbA} &\OK{\LbF} &\im{\LbE} 
    &\scU
    &\im{\LrA} &\im{\LrC} &\OK{\LbF} &\im{\LaF}
    \\\hdashline
    \OK{\LbG} &\im{\LbA} &\ok{\LbF} &\im{\LbE} 
    &\scUC
    &\OK{\LbG} &\im{\LbA} &\ok{\LbF} &\im{\LaF}
    \\\hdashline
    \im{\LrA} &\im{\LbA} & \ok{\LbF} &\im{\LbE}         
    &\scF
    &\im{\LrA} &\im{\LrC} &\ok{\LbF} &\im{\LaF}
    \\\hdashline
    \ok{\LbG} &\im{\LbA} & \ok{\LbF} &\IM{\LbE}        
    &\scFC
    &\ok{\LbG} &\im{\LbA} &\ok{\LbF} &\im{\LaF}
    \\\hdashline
    \im{\LrA} &\im{\LbA} &\ok{\LbF} &\ok{\LbD} 
    &\scFkB
    &\im{\LrA} &\im{\LbA} &\ok{\LbF} &\OK{\LbD}   
    \\\hdashline
    \IM{\LrA} &\im{\LbA} &\ok{\LbF} &\ok{\LbD} 
    &\scFBR
    &\IM{\LrA} &\im{\LbA} &\ok{\LbF} &\ok{\LbD}
    \\\hdashline
    \ok{\LbG} &\IM{\LbA} &\ok{\LbF} &\ok{\LbD} 
    &\scFCBR
    &\ok{\LbG} &\IM{\LbA} &\ok{\LbF} &\ok{\LbD}
    \\\hline
  \end{tabular}
}
\end{table*}

All tables are designed to be read as follows:
Each row represents a different scheduler, while columns distinguish other assumptions, such as multiplicity, conditions on the number of robots $n$, conditions on the maximum number of faulty robots $f$, or whether deterministic or probabilistic solutions are admissible.

Each cell answers whether the problem admits a solution under the corresponding set of assumptions.
A positive result appears as ``OK'' followed by the number of the corresponding lemma or theorem in brackets.
Conversely, a negative result (impossibility) is denoted by ``NO'' and a greyed background. 

An ``\textbf{OK}'' or ``\textbf{NO}'' in bold means that the cell corresponds to the assumptions stated explicitly
in the relevant theorem.
When the text appears in normal face, the result comes instead as a consequence of the theorem and the relationship between assumptions.
For instance, a positive result expressed and proved with an unfair centralized scheduler (e.g., Table~\ref{tab:sum:none}; \LaC) necessarily
applies to the more restrictive schedulers, such as the fair centralized or round-robin schedulers, even though this is implicit.

\subsection{Strong self-stabilizing gathering}

The results pertaining to the strong gathering in a fault-free model are summarized in Table~\ref{tab:sum:none},
while those related to the crash model with a single faulty robot are in Table~\ref{tab:sum:strong-crash}.
The tables are divided vertically according to the availability of multiplicity detection,
then whether gathering is deterministic or probabilistic, and finally to the number of robots $n$ (i.e., $n=2$ or $n>2$).
Note that, when $n=2$, the detection of multiplicity is irrelevant, and thus the results are identical in both columns.

\subsubsection{Fault-free model}
As shown on Table~\ref{tab:sum:none}, in the absence of multiplicity detection,
a bounded scheduler is both necessary and sufficient for solving probabilistic gathering of more than two robots.
There is however no deterministic solution, regardless of the scheduler (i.e., even if the scheduler is round-robin). 

In the presence of multiplicity detection, gathering is known to be possible with a fair scheduler,
as proved by Suzuki and Yamashita \cite{SY99}. The question remains open in the case of unfair schedulers.

When there are only two robots, gathering is known to be more difficult than with three  or more robots, since all configurations are symmetrical.
Suzuki and Yamashita \cite{SY99} have proved the impossibility under a fair scheduler, and their proof actually applies to more restrictive schedulers,
such as the fair 1-bounded scheduler.
Interestingly, the problem becomes solvable under all classes of \emph{centralized} schedulers, even the unfair ones.




\subsubsection{Crash model}
Table~\ref{tab:sum:strong-crash} summarizes the results obtained for the strong gathering problem with at most one robot crash.

Interestingly, without multiplicity detection, the results obtained for the fault-free and the crash models are identical,
although they are covered by different theorems.
Unlike in the fault-free model, multiplicity detection does not seem to help solve gathering. Indeed, in the crash model,
results are identical whether or not robots are able to detect multiplicity, whereas they differed widely in the fault-free case.

In other words, while the introduction of multiplicity detection is indeed determinant in the fault-free case,
it has no effect on solvability when faced with a single crashed robot.

\begin{table*}
  \centering
  \caption{Weak gathering problem}
  \label{tab:sum:weak}
\subfloat[Crash model]{
 \label{tab:sum:weak-crash}
  \begin{tabular}{|c;{2pt/2pt}c|c;{2pt/2pt}c||c||c;{2pt/2pt}c|c;{2pt/2pt}c|}
    \hline
    \multicolumn{4}{|c||}{$f=1$}
    &
    & \multicolumn{4}{c|}{$2\leq f < n$}
    \\\cline{1-4}\cline{6-9}
    \multicolumn{2}{|c|}{multiplicity}    & \multicolumn{2}{c||}{without multiplicity}
    & \textbf{Scheduler}
    & \multicolumn{2}{c|}{multiplicity}    & \multicolumn{2}{c|}{without multiplicity}
    \\\cline{1-4}\cline{6-9}
    {determ.} &  {proba.} & {determ.} &  {proba.}
    &
    & {determ.} &  {proba.} & {determ.} &  {proba.} 
    \\\hline
                 &                &\im{\LrC} &\im{\LaF}
    &\scU
    &            &                &\im{\LcA} &\im{\LcA}
    \\\hdashline
                 &                &\im{\LaH} &\im{\LaF}
    &\scUC
    &            &                &\im{\LcA} &\im{\LcA}
    \\\hdashline
    \DUB{\OK{\LrD}$^b$} 
                 &\OK{\LcC} &\im{\LrC} &\im{\LaF}
    &\scF
    &\IM{\LcD}   &\OK{\LcC} &\im{\LcA} &\im{\LcA}
    \\\hdashline
    \OK{\LcB} &\ok{\LcB} &\im{\LaH} &\im{\LaF}
    &\scFC
    &\OK{\LcB} &\ok{\LcB} &\im{\LcA} &\im{\LcA}
    \\\hdashline
    \DUB{\ok{\LrD}} &\ok{\LcC} &\im{\LaH} &\ok{\LbD}
    &\scFkB
    &\IM{\LcD}   &\ok{\LcC} &\im{\LcA} &\im{\LcA}
    \\\hdashline
    \DUB{\ok{\LrD}} &\ok{\LcC} &\spec{\LaA}$^a$ &\ok{\LbD}
    &\scFBR
    &\IM{\LcD}   &\ok{\LcC} &\im{\LcA} &\im{\LcA}
    \\\hdashline
    \ok{\LcB} &\ok{\LcB} &\spec{\LaA}$^a$ &\ok{\LbD}
    &\scFCBR
    &\ok{\LcB} &\ok{\LcB} &\IM{\LcA} &\IM{\LcA}
    \\\hline
    \multicolumn{9}{c}{
			\begin{minipage}{.95\textwidth}
	    \footnotesize
	    \begin{trivlist}
	      \item[$^a$]{
	      	Special: Th.~\ref{th:detimpossibility-ff} proves the impossibility of self-stabilizing gathering.
					Th.~\ref{th:detimpossibility-ff:distinct} proves it for gathering \emph{provided Conjecture~\ref{cj:criteria} holds}.
				}
	      \item[$^b$]{
	      	Note that the results derived from Theorem~\ref{th:ap:crashpos} hold for the case (3,1).
		 			According to Note \ref{note:imp:crash:(n,1)}, in the case of $(n,1)$-crash, weak gathering is
					possible only if, during the execution, each configuration has at most one multiplicity point.
					Therefore, the self-stabilizing $(n,1)$ weak-gathering is impossible since the initial
					configuration can contain more than one multiplicity point.
				}
	    \end{trivlist}
	  	\end{minipage}
    }
  \end{tabular}
}

\subfloat[Byzantine model]{
  \label{tab:sum:Byz}
  \begin{tabular}{|l@{}||c:c:c;{2pt/2pt}c:c:c|}
    \hline
    & \multicolumn{6}{c|}{multiplicity; deterministic}
    \\\cline{2-7}
    & \multicolumn{3}{c;{2pt/2pt}}{$f=1$} &  \multicolumn{2}{c|}{$2 \leq f < n/2$} 
    \\\cline{2-7}
    &                 &$n\geq 4$ & $n\geq 4$& $n\geq 4$ & $n\geq 4$    
    \\
    \textbf{Scheduler}
    & $n=3$           & (even)   & (odd)    & (even)   & (odd)
    \\\hline\hline
    \scU\tnote{a}
    &\im{\LrE}        &\im{\LdH} &          &\im{\LdH} &\im{\LdI}
    \\\hdashline
    \scUC\tnote{a}
    & \im{\LdC}       &\im{\LdH} &          &\im{\LdH} &\im{\LdI}
    \\\hdashline
    \scF
    &\IM{\LrE}        &\im{\LdH} &          &\im{\LdH} &\im{\LdI}
    \\\hdashline
    \scFC
    & \im{\LdC}       &\im{\LdH} &          &\im{\LdH} &\im{\LdI}
    \\\hdashline
    \scFkB
    & \im{\LdC}       &\im{\LdH} &          &\im{\LdH} &\im{\LdI}
    \\\hdashline
    $(k\geq n-1)$-bounded
    & \im{\LdC}       &\im{\LdH} &           &\im{\LdH} &\im{\LdI}
    \\\hdashline
    $(\Gamma(n,f)\leq k\leq n-2)$-bounded
    & \im{\LdC}       &\im{\LdH} &           &\im{\LdH} &\im{\LdI}
    \\\hdashline
    $(2\leq k<\Gamma(n,f))$-bounded
    & \im{\LdC}       &          &          &          &          
    \\\hdashline
    \scFBR
    &                 &          &          &          &          
    \\\hdashline
    centr. $(k\geq n-1)$-bound.
    &  \im{\LdC}      &\IM{\LdH} &          &\IM{\LdH} &\IM{\LdI}
    \\\hdashline
    centr. $(\Gamma(n,f)\leq k\leq n-2)$-bound.
    &  \im{\LdC}      &\IM{\LdH} &          &\IM{\LdH} &\IM{\LdI}
    \\\hdashline
    centr. $(2\leq k<\Gamma(n,f))$-bound.
    &  \IM{\LdC}      &          &          &          &          
    \\\hdashline
    \scFCBR
    &                 &          &          &          &          
    \\\hdashline
    	\scFSyn
		&  \OK{\LrG}      &\OK{\LrF} &\OK{\LrF} &\multicolumn{2}{c|}{\OK{\LrF} if $n\geq 3f+1$}
    \\\hline
    \multicolumn{7}{l}{%
    	$\Gamma(n,f)=\left\{
				\left\lceil\frac{n-f}{f}\right\rceil \mbox{ if $n$ even;}
		  				\left\lceil\frac{n-f}{f-1}\right\rceil \mbox{ if $n$ odd}
				\right\}
			$
    	}
		\\
  \end{tabular}
}
\end{table*}

\subsection{Weak self-stabilizing gathering}
Table~\ref{tab:sum:weak} summarizes the results for weak gathering.
Let us first remind that, in the fault-free model, there is actually no difference between strong and weak gathering
(since the only difference in definitions is about the requirements put on the faulty robots),
and thus the results of Table~\ref{tab:sum:none}, although not repeated, are of course also relevant here.

Table~\ref{tab:sum:weak-crash} summarizes the results for weak gathering (i.e., only the correct robots are required to gather at the same location)
and distinguishes between the case of a single crash and multiple crashes.
One interesting observation is that, in the case of a single crash (left part of Table~\ref{tab:sum:weak-crash}),
results of weak with respect to strong gathering differ only if robots are able to detect multiplicity.
In particular, Theorems~\ref{th:ft-det-gathering} and~\ref{th:wg} show that weak gathering is possible with schedulers for which strong gathering is not.
This is because a system may reach a stable configuration in which all robots except the faulty one share the same location.
In such a configuration, weak gathering is achieved but strong gathering is not.

In the case of multiple crashes and without multiplicity detection, even probabilistic gathering is impossible under any of the schedulers considered.
With multiplicity detection and fair schedulers, however, probabilistic gathering is possible under any fair scheduler while
deterministic gathering is possible if and only if the scheduler is also centralized.
The question remains open for unfair schedulers, but we believe that the answer depends greatly on minute details in the definition of the unfair scheduler.

\subsection{Byzantine model}
While Byzantine gathering is possible in fully synchronous
environments, other positive results remain quite elusive. We have been able
to extend impossibility results, but unable to find additional solutions for
other models.

Under very specific assumptions, Algorithm~\ref{alg:ft-gathering} is likely to solve
Byzantine gathering for some values of $f$, $n$, and $k$. However, this requires very
specific assumptions, among which the requirement that \emph{Byzantine robots have a
mobility range no larger than the correct ones}.
We have found a counter-example where the algorithm fails without this assumption,
and thus omitted entirely from the study, thus leaving the question open.

%% file: conclusion.tex
\section{Conclusion}
\label{conclusion}

The results presented in this paper extend prior work on the self-stabilizing
gathering problem in fault-free and fault-prone environments, by shading light
on the subtil line between possibility and impossibility. Most notably, we
identify the role that additional synchrony, embodied by schedulers, can play
toward making the problem possible.
So far, our work is the most extensive study on the combined roles that
randomization, multiplicity, and schedulers (centralized and bounded) can play
in allowing a solution to fault-free and fault-tolerant gathering.

In particular, we have strengthened several key impossibility
results on gathering, including Prencipe's~\cite{Pre05} impossibility of
fault-free gathering in the absence of multiplicity strengthened to cover
up to the round-robin or 2-bounded centralized schedulers (depending on the
definition of the problem), and Agmon and Peleg's~\cite{AP06} impossibility
of Byzantine gathering under a fair scheduler extended to cover bounded
centralized schedulers.


The main results of the paper are summarized in Table~\ref{tab:sum:none} for
fault-free systems; in Table~\ref{tab:sum:strong-crash} and Table~\ref{tab:sum:weak-crash}
for strong, resp. weak, gathering in crash-prone systems; and in Table~\ref{tab:sum:Byz}
for weak gathering problem in Byzantine-prone systems.

The main results of the paper are summed up in Table~\ref{tab:sum:none} for fault-free systems; in Table~\ref{tab:sum:strong-crash} and Table~\ref{tab:sum:weak-crash} for strong respectivelly weak gathering in crash-prone systems; and in Table~\ref{tab:sum:Byz} for the weak gathering problem in Byzantine-prone systems.

%% file: appendix.tex

\appendix
\section{Appendix}
\subsection{Necessity of the Side Move for Algorithm~\ref{alg:ft-gathering}}
\label{apx:need:side-move}
We must now show the necessity of introducing a side move in Algorithm~\ref{alg:ft-gathering}.
Assuming that robots execute the naive algorithm (Algorithm~\ref{alg:ft-gathering} without
the clause executing the side move), we exhibit a situation in which the robots are unable
to gather (depicted in Fig.~\ref{fig:2castles}):

\tikzset{%
	robot/.style={
		circle,
		inner sep=0pt,
		minimum size=5pt,
		fill=black,
		draw=black
	},
	active_robot/.style={
		robot,
		fill=red!50,
		draw=red,
		thick
	},
	configuration/.style={
		draw=gray,
		rectangle,
		rounded corners,
		minimum height=1.4cm,
		minimum width=2cm,
		very thin
	},
	transition/.style={
		->,
		>=latex,
		draw=black,
		rounded corners
	},
	move/.style={
		->,
		>=stealth,
		densely dashed,
		thick,
		red
	},
	conf/.style={
		draw=black,
		ellipse,
		thin
	},
	trans/.style={
		->,
		>=stealth,
		looseness=.5,
		auto,
		draw=black
	},
	point/.style={
		draw=gray,
		cross out,
		scale=0.75,
		minimum height=1pt,
		minimum width=1pt,
		ultra thin		
	}
}

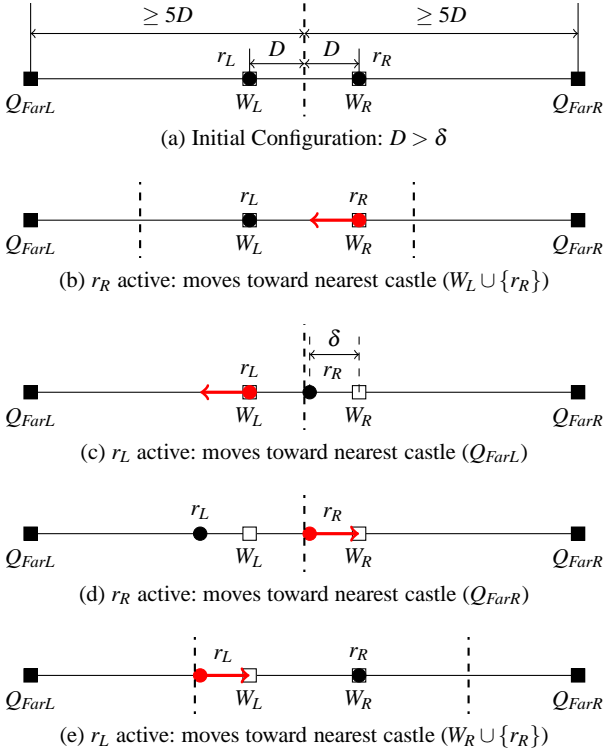
\begin{figure}
\subfloat[Initial Configuration: $D>\delta$]{
	\label{fig:2castles:initial}
	\begin{tikzpicture}
		\def\vdelta{0.65}
		\def\epsilon{0}
		\def\step{0.72}
		\coordinate (A) at (-5*\step, 0);
		\coordinate (B) at (5*\step, 0);
		
		\coordinate (L) at (-\step, 0);
		\coordinate (R) at ( \step, 0);
		
		\coordinate (Lr) at (-\step-\vdelta, 0);
		\coordinate (Rr) at ( \step-\vdelta, 0);

		\clip ($(A)+(-0.5,-0.5)$) rectangle ($(B)+(0.5,1)$);
		
		\draw (A) -- (B);
		
		\node[robot,rectangle,label=below:{$\castle{Q}_\mathit{FarL}$}] () at (A) {};
		\node[robot,rectangle,label=below:{$\castle{Q}_\mathit{FarR}$}] () at (B) {};
		\node[robot,rectangle,fill=white,label=below:{$W_L$}] () at (L) {};
		\node[robot,rectangle,fill=white,label=below:{$W_R$}] () at (R) {};

		\node[robot,label=above left:{$r_L$}] () at (L) {};
		\node[robot,label=above right:{$r_R$}] () at (R) {};

		\draw[thick,dashed] ($(A)!.5!(B)$) +(0,1) -- +(0,-0.5);
		\draw[thin] (R) -- +(0,0.4);
		\draw[thin] (L) -- +(0,0.4);
		\draw[thin] (A) -- +(0,1.1);
		\draw[thin] (B) -- +(0,1.1);
		\draw[<->] ($(R)+(0,0.2)$) -- node[sloped,above]{$D$} (0,0.2);
		\draw[<->] ($(B)+(0,0.6)$) -- node[sloped,above]{$\geq 5D$} (0,0.6);
		\draw[<->] ($(L)+(0,0.2)$) -- node[sloped,above]{$D$} (0,0.2);
		\draw[<->] ($(A)+(0,0.6)$) -- node[sloped,above]{$\geq 5D$} (0,0.6);
	\end{tikzpicture}
}

\subfloat[$r_R$ active: moves toward nearest castle ($W_L \cup \{r_R\}$)]{
	\label{fig:2castles:step1}
	\begin{tikzpicture}
		\def\vdelta{0.65}
		\def\epsilon{0}
		\def\step{0.72}
		\coordinate (A) at (-5*\step, 0);
		\coordinate (B) at (5*\step, 0);
		
		\coordinate (L) at (-\step, 0);
		\coordinate (R) at ( \step, 0);
		
		\coordinate (Lr) at (-\step-\vdelta, 0);
		\coordinate (Rr) at ( \step-\vdelta, 0);

		\clip ($(A)+(-0.5,-0.5)$) rectangle ($(B)+(0.5,0.5)$);
		
		\draw (A) -- (B);
		
		\node[robot,rectangle,label=below:{$\castle{Q}_\mathit{FarL}$}] () at (A) {};
		\node[robot,rectangle,label=below:{$\castle{Q}_\mathit{FarR}$}] () at (B) {};
		\node[robot,rectangle,fill=white,label=below:{$W_L$}] () at (L) {};
		\node[robot,rectangle,fill=white,label=below:{$W_R$}] () at (R) {};

		\draw[very thick,red,->] (R) -- (Rr);

		\node[robot,label=above:{$r_L$}] () at (L) {};
		\node[robot,red,label=above:{$r_R$}] () at (R) {};

		\draw[thick,dashed] ($(A)!.5!(L)$) +(0,1) -- +(0,-1);
		\draw[thick,dashed] ($(L)!.5!(B)$) +(0,1) -- +(0,-1);
	\end{tikzpicture}
}

\subfloat[$r_L$ active: moves toward nearest castle ($Q_\mathit{FarL}$)]{	
	\label{fig:2castles:step2}
	\begin{tikzpicture}
		\def\vdelta{0.65}
		\def\epsilon{0}
		\def\step{0.72}
		\coordinate (A) at (-5*\step, 0);
		\coordinate (B) at (5*\step, 0);
		
		\coordinate (L) at (-\step, 0);
		\coordinate (R) at ( \step, 0);
		
		\coordinate (Lr) at (-\step-\vdelta, 0);
		\coordinate (Rr) at ( \step-\vdelta, 0);

		\clip ($(A)+(-0.5,-0.5)$) rectangle ($(B)+(0.5,0.9)$);
		
		\draw (A) -- (B);
		
		\node[robot,rectangle,label=below:{$\castle{Q}_\mathit{FarL}$}] () at (A) {};
		\node[robot,rectangle,label=below:{$\castle{Q}_\mathit{FarR}$}] () at (B) {};
		\node[robot,rectangle,fill=white,label=below:{$W_L$}] () at (L) {};
		\node[robot,rectangle,fill=white,label=below:{$W_R$}] () at (R) {};

		\draw[very thick,red,->] (L) -- (Lr);

		\node[robot,red,label=above:{$r_L$}] () at (L) {};
		\node[robot,label=above right:{$r_R$}] () at (Rr) {};

		\draw[dashed] (R) -- +(0,0.75);
		\draw[dashed] (Rr) -- +(0,0.75);
		\draw[<->] ($(R)+(0,0.5)$) -- node[sloped,above]{$\delta$} ($(Rr)+(0,0.5)$);
		
		\draw[thick,dashed] ($(A)!.5!(B)$) +(0,1) -- +(0,-1);
	\end{tikzpicture}
}
	
\subfloat[$r_R$ active: moves toward nearest castle ($Q_\mathit{FarR}$)]{	
	\label{fig:2castles:step3}
	\begin{tikzpicture}
		\def\vdelta{0.65}
		\def\epsilon{0}
		\def\step{0.72}
		\coordinate (A) at (-5*\step, 0);
		\coordinate (B) at (5*\step, 0);
		
		\coordinate (L) at (-\step, 0);
		\coordinate (R) at ( \step, 0);
		
		\coordinate (Lr) at (-\step-\vdelta, 0);
		\coordinate (Rr) at ( \step-\vdelta, 0);

		\clip ($(A)+(-0.5,-0.5)$) rectangle ($(B)+(0.5,0.5)$);
		
		\draw (A) -- (B);
		
		\node[robot,rectangle,label=below:{$\castle{Q}_\mathit{FarL}$}] () at (A) {};
		\node[robot,rectangle,label=below:{$\castle{Q}_\mathit{FarR}$}] () at (B) {};
		\node[robot,rectangle,fill=white,label=below:{$W_L$}] () at (L) {};
		\node[robot,rectangle,fill=white,label=below:{$W_R$}] () at (R) {};

		\draw[very thick,red,->] (Rr) -- (R);

		\node[robot,label=above:{$r_L$}] () at (Lr) {};
		\node[robot,red,label=above right:{$r_R$}] () at (Rr) {};

		\draw[thick,dashed] ($(A)!.5!(B)$) +(0,1) -- +(0,-1);
	\end{tikzpicture}
}

\subfloat[$r_L$ active: moves toward nearest castle ($W_R \cup \{r_R\}$)]{	
	\label{fig:2castles:step4}
	\begin{tikzpicture}
		\def\vdelta{0.65}
		\def\epsilon{0}
		\def\step{0.72}
		\coordinate (A) at (-5*\step, 0);
		\coordinate (B) at (5*\step, 0);
		
		\coordinate (L) at (-\step, 0);
		\coordinate (R) at ( \step, 0);
		
		\coordinate (Lr) at (-\step-\vdelta, 0);
		\coordinate (Rr) at ( \step-\vdelta, 0);

		\clip ($(A)+(-0.5,-0.5)$) rectangle ($(B)+(0.5,0.5)$);
		
		\draw (A) -- (B);
		
		\node[robot,rectangle,label=below:{$\castle{Q}_\mathit{FarL}$}] () at (A) {};
		\node[robot,rectangle,label=below:{$\castle{Q}_\mathit{FarR}$}] () at (B) {};
		\node[robot,rectangle,fill=white,label=below:{$W_L$}] () at (L) {};
		\node[robot,rectangle,fill=white,label=below:{$W_R$}] () at (R) {};

		\draw[very thick,red,->] (Lr) -- (L);

		\node[robot,red,label=above right:{$r_L$}] () at (Lr) {};
		\node[robot,label=above:{$r_R$}] () at (R) {};

		\draw[thick,dashed] ($(A)!.5!(R)$) +(0,1) -- +(0,-1);
		\draw[thick,dashed] ($(R)!.5!(B)$) +(0,1) -- +(0,-1);
	\end{tikzpicture}
}
	\caption{Illustration of the necessity of introducing the side move. The naive algorithm (without side move) can result in and endless cycle.}
	\label{fig:2castles}
\end{figure}

Consider the initial configuration depicted in Figure~\ref{fig:2castles:initial}.
Assume that the reachable distance of all robots is the same and call it $\delta$.
Let $D$ be some arbitrary distance strictly larger than $\delta$.
The robots (or a subset thereof) are initially located such that they form four castles on a segment.
Let $\castle{Q}_\mathit{FarL},\castle{Q}_\mathit{FarR}$ be the two castles at both ends of the segment and assume that they consist only of crashed robots.
Let $W_L,W_R$ be two towers such that they become a castle by adding one robot. For simplicity, assume again that they also consist only of crashed robots.
Let $r_L,r_R$ be two correct robots initially with $W_L,W_R$ respectively.
The location of the four castles is symmetric such that the midpoint between $\castle{Q}_\mathit{FarL}$ and $\castle{Q}_\mathit{FarR}$ is also the midpoint between $W_L$ and $W_R$.
The distance between $W_L$ and $W_R$ is $2D$ and the distance between $\castle{Q}_\mathit{FarL}$ and $\castle{Q}_\mathit{FarR}$ is at least $10D$.

Consider the scheduler as an adversary following a round-robin policy.
First, $r_R$ is active (Fig.~\ref{fig:2castles:step1}).
According to the naive algorithm, $r_R$ must move toward the nearest castle, which is the castle formed by $W_L$ and $r_L$.
The dashed lines on the figure represent the boundaries of the Voronoi cells of each of the three castles: $\left\{\castle{Q}_\mathit{FarL}, \castle{Q}_\mathit{FarR}, W_L\cup\left\{r_L\right\} \right\}$.
Since $r_R$ is located inside the Voronoi cell of castle $W_L\cup\left\{r_L\right\}$, it moves toward it.

Second, $r_L$ is active (Fig.~\ref{fig:2castles:step2}).
Since $r_R$ has moved in the previous step, $W_R$ is no longer the location of a castle.
Now, $r_L$ is located inside the Voronoi cell of $\castle{Q}_\mathit{FarL}$ and moves toward it.

Third, $r_R$ is active again (Fig.~\ref{fig:2castles:step3}).
There are only two castles left on the configuration, namely $\castle{Q}_\mathit{FarL}$ and $\castle{Q}_\mathit{FarR}$.
Since $D>\delta$, $r_R$ is located to the right of the midpoint between $W_L$ and $W_R$,
which is also the midpoint between $Q_\mathit{FarL}$ and $\castle{Q}_\mathit{FarR}$.
This means that $r_R$ is in the Voronoi cell of $\castle{Q}_\mathit{FarR}$ and hence moves toward it.
But, because $r_R$ is at distance $\delta$ to $W_R$, it ends its movement exactly at $W_R$, thus forming a castle again.

Fourth, $r_L$ is active and there are three castles (Fig.~\ref{fig:2castles:step4}).
By construction, $W_R$ is located at a distance at least $6D$ from $\castle{Q}_\mathit{FarL}$,
and hence $r_L$ remains inside the Voronoi cell of the castle formed by $W_R$ and $r_R$.
$r_L$ is also at a distance $\delta$ from $W_L$, and hence ends its movement exactly at $W_L$, forming a castle.
This leads back to the initial configuration (Fig.~\ref{fig:2castles:initial}), and thus the cycle continues forever.


\subsection{Disambiguation of Side Move for Algorithm~\ref{alg:ft-prob-gathering}}
\label{sec:side-move:disamb}

\begin{algorithm}
	\caption{Disambiguation of side move (robot at \point{p}).}
  \label{alg:sidemove}
  \begin{small}
  	\textbf{Procedure}:\\
		\IND $\textsc{SideMove}(\point{p},\point{q},\multiset{\Pi})$ $\longrightarrow$ \\
		\IND\IND set origin at \point{q}\\
		\IND\IND let $\multiset{\Pi}_\point{q} \subset \multiset{\Pi}$ be all robots in \Vcell{\point{q}} or on its boundary.\\
		\IND\IND let $r_\textsc{cw}$ be the first robot clockwise in $\multiset{\Pi}_\point{q}$ starting from \point{p}.\\
		\IND\IND let $\theta_\textsc{cw}$ be angle $\angle \point{p}\point{q}r_\textsc{cw}$ or $\pi$, whichever is smaller.\\
		\IND\IND let $\theta^+$ be one third of $\theta_\textsc{cw}$.\\
		\IND\IND let $\point{v}_\point{p}$ be the intersection of $\overline{\point{q}\point{p}}$ and the boundary of \Vcell{\point{q}}.\\
		\IND\IND let $\mathit{ray}$ be the ray from \point{q} with clockwise angle $\theta^+$ from $\overline{\point{q}\point{p}}$.\\
		\IND\IND let $\point{V}_a$ be the intersection of $\mathit{ray}$ with \Vcell{\point{q}}.\\
		\IND\IND let $\point{V}_b$ be the intersection of $\mathit{ray}$ with the circumference of $D$.\\
		\IND\IND let $\point{V'}$ be $\point{V}_a$ or $\point{V}_b$, whichever is nearest \point{q}.\\
		\IND\IND let vector $\vec{\mathit{target}} = \frac{\dist{\point{q}}{\point{p}}}{\dist{\point{q}}{\point{v}_\point{p}}} \vec{\point{q}\point{v'}}$.\\
		\IND\IND move toward point $\vec{\mathit{target}}$.
	\end{small}
\end{algorithm}

\begin{figure}
	\centering
	\begin{tikzpicture}
		\def\horizK{6}
		\def\angle{40}
		\def\halfalpha{2}
	
		\clip (-.5,-.5) rectangle (2*\halfalpha,3*\halfalpha+0.5);
		
		\coordinate (q) at (0,0);
		\coordinate (p) at (0,2*\halfalpha);
		\coordinate (vp) at (0,3*\halfalpha);
		\coordinate (k) at (\horizK,0);
		\coordinate (farV) at (\angle:8);
		
		\draw[name path=D]   (0,\halfalpha) circle (\halfalpha);
		\draw[name path=ray,dotted] ($(q)!-0.1!(farV)$) -- (farV);
		\draw[name path=Vcell,thick] ($(vp)!-0.1!(k)$) -- ($(vp)!1.1!(k)$);
		\draw[dotted] ($(q)!-.1!(vp)$) -- ($(q)!1.1!(vp)$);
		\draw (q) -- node[left] {$a$} (p) -- node[left] {$b$} (vp);
		\draw[dotted] (p) arc (90:0:2*\halfalpha) ;
		
		\draw (90:0.5) arc (90:\angle:0.5) node[above,sloped] {$\theta^+$};
		
		\coordinate[name intersections={of=ray and Vcell,by=Va}];
		\coordinate[name intersections={of=ray and D,by={dummy,Vb}}];

		\coordinate (p') at ($(q)!0.66667!(Vb)$);
		
		\node[robot,label=below:{\castle{Q}}] (nq) at (q) {};
		\node[robot,label=above left:{\point{P}}] (np) at (p) {};
		\node[cross out,draw] (nvp) at (vp) {};
		\node[cross out,draw] (nk) at (k) {};

		\node[cross out,draw,label=right:{$\point{V}_a$}] at (Va) {};
		\node[cross out,draw,label=right:{$\point{V}_b$}] at (Vb) {};
		\draw(vp) -- node[above,sloped] {\Vcell{q} : $y=mx+a+b$} (Va);

		\draw (q) -- node[below,sloped] {$a'$} (p') -- node[below,sloped] {$b'$} (Vb);

		\node[robot,red,fill=white,text=black,label=left:{\point{P'}}] (np') at (p') {};
		\draw[move] (np) -- (np');		
	\end{tikzpicture}
	\caption{Disambiguated side move. $\frac{a}{a+b} = \frac{a'}{a'+b'}$.}
	\label{fig:side-move:extended}
\end{figure}
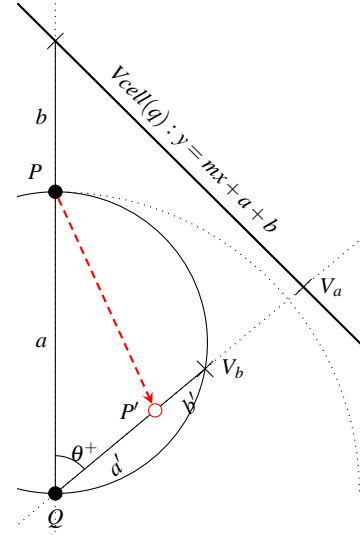

Algorithm~\ref{alg:sidemove} describes one method to disambiguate the side move, and is
illustrated in Figure~\ref{fig:side-move:extended}.
The lengths $a$ and $b$ depend on the position of robot \point{P} on the ray
from castle \castle{Q} relative to the boundary of the Voronoi cell of \castle{Q}.

The construction uses a trisection of the sector $S$ calculated in the original side move.
This ensures that a robot moving from a different ray does not end up at the same location.

In addition, taking the minimum between points $\point{V}_a$ and $\point{V}_b$ ensures that the segment
$\overline{\castle{Q}\point{V'}}$ lies entirely within the zone desired for a side move. Since the zone is
convex (intersection of three convex areas), segment $\point{P}\point{V'}$ lies entirely inside
the zone.

\begin{eqnarray}
	\alpha &=& \frac{a}{a+b} \in \left(0;1\right]
	\\
	a'(\alpha) &=&
	 \min \left(
		\alpha^2 \cos\theta^+
	,
		\frac{
			\alpha
		}{
			\cos\theta^+
			- \frac{1}{m}\sin\theta^+
		}
	\right)
\end{eqnarray}

Since $a'(\alpha)$ is taken as the minimum of two functions that are
both monotonic increasing in $\alpha$ over the range considered, $a'(\alpha)$
is itself monotonic increasing. It follows that, for two values $\alpha_1$ and
$\alpha_2$ with $\alpha_1\not=\alpha_2$,
the segments from $P(\alpha)$ to $P'(\alpha)$ do not cross.